%% file: main.tex
\documentclass[10pt,twocolumn,letterpaper]{article}

\usepackage[pagenumbers]{iccv} 

\input{preamble}

\definecolor{iccvblue}{rgb}{0.21,0.49,0.74}
\usepackage[pagebackref,breaklinks,colorlinks,allcolors=iccvblue]{hyperref}

\def\paperID{7948} 
\def\confName{ICCV}
\def\confYear{2025}

\title{SteerX: Creating Any Camera-Free 3D and 4D Scenes with Geometric Steering}

\author{
    Byeongjun Park\textsuperscript{\rm1,\rm2}\thanks{Equal contribution, $^\dagger$Corresponding author}  \qquad
    Hyojun Go\textsuperscript{\rm2}\footnotemark[1] \qquad
    Hyelin Nam\textsuperscript{\rm2} \qquad
    Byung-Hoon Kim\textsuperscript{\rm2,\rm3} \\ 
    Hyungjin Chung\textsuperscript{\rm2$\dagger$} \qquad
    Changick Kim\textsuperscript{\rm1$\dagger$} \vspace{0.45em} \\
    \textsuperscript{\rm 1} KAIST \qquad
    \textsuperscript{\rm 2} EverEx \qquad
    \textsuperscript{\rm 3} Yonsei University
    \vspace{0.45em} \\
    \href{https://byeongjun-park.github.io/SteerX/}{https://byeongjun-park.github.io/SteerX/}
}

\begin{document}

\maketitle
\input{sec/0_abstract}    
\input{sec/1_introduction}

\input{sec/2_related_work}

\input{sec/3_preliminary}
\input{sec/4_method}
\input{sec/5_experiments}
\input{sec/6_conclusion}

{
    \small
    \bibliographystyle{ieeenat_fullname}
    \bibliography{main}
}

\maketitlesupplementary
\input{sec/_appendix}

\end{document}

%% file: preamble.tex
\newcommand{\add}[1]{{\color{blue}#1}}

\usepackage{animate}
\usepackage{graphicx}	
\usepackage{amsmath}	
\usepackage{amssymb}	
\usepackage{amsthm}
\usepackage{booktabs}
\usepackage{cuted}
\usepackage{times}
\usepackage{epsfig}
\usepackage{caption}
\usepackage{float}
\usepackage{placeins}
\usepackage{color, colortbl}
\usepackage{stfloats}
\usepackage{enumitem}
\usepackage{tabularx}
\usepackage{xstring}
\usepackage{multirow}
\usepackage{xspace}
\usepackage{url}
\usepackage{subcaption}
\usepackage[hang,flushmargin]{footmisc}
\usepackage{kotex}
\usepackage{arydshln}
\usepackage{adjustbox}
\usepackage{wrapfig}
\usepackage{algorithm,algorithmicx,algpseudocode}
\usepackage{listings}
\usepackage{bm}

\newlength\paramargin
\newlength\abovetabcapmargin
\newlength\belowtabcapmargin
\newlength\abovefigcapmargin
\newlength\belowfigcapmargin
\newlength\aboveeqmargin
\newlength\beloweqmargin

\usepackage{amsmath}
\DeclareMathOperator*{\argmax}{arg\,max}

\def\x{{\mathbf{x}}}

\def\tildep{{\tilde{p}_\theta}}

\usepackage{thmtools,thm-restate}

%% file: sec/0_abstract.tex
\input{figure/teasure}

\begin{abstract}



Recent progress in 3D/4D scene generation emphasizes the importance of physical alignment throughout video generation and scene reconstruction. However, existing methods improve the alignment separately at each stage, making it difficult to manage subtle misalignments arising from another stage. Here, we present SteerX, a zero-shot inference-time steering method that unifies scene reconstruction into the generation process, tilting data distributions toward better geometric alignment. To this end, we introduce two geometric reward functions for 3D/4D scene generation by using pose-free feed-forward scene reconstruction models. Through extensive experiments, we demonstrate the effectiveness of SteerX in improving 3D/4D scene generation.

\end{abstract}

%% file: figure/teasure.tex
\begin{strip}
    \centering
    \vspace{-5.55em}
    \resizebox{\textwidth}{!}{
    \includegraphics[width=\textwidth]{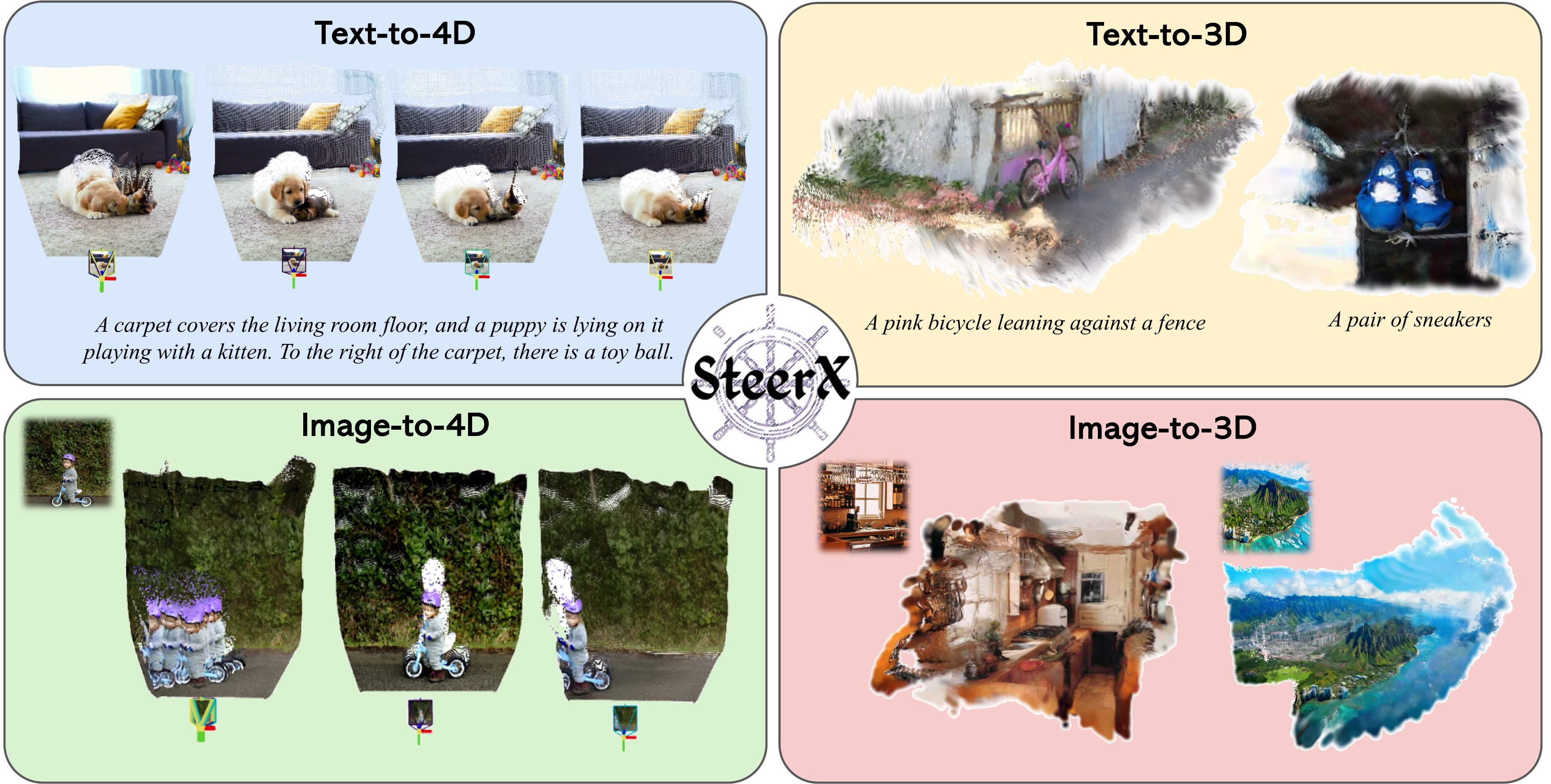}  
    }
    \vspace{\abovefigcapmargin}
    \vspace{-0.5cm}
    \captionof{figure}{\textbf{SteerX} is a zero-shot inference-time steering approach that seamlessly integrates video generative models~\cite{genmo2024mochi, kong2024hunyuanvideo, yang2024cogvideox, sun2024dimensionx, go2024splatflow} and feed-forward scene reconstruction models~\cite{tang2024mv, zhang2024monst3r, go2024splatflow}, enabling any 3D and 4D scene generation without explicit camera conditions.}
    \label{fig:teasure}
    \vspace{-0.1cm}
\end{strip}

%% file: sec/1_introduction.tex
\section{Introduction}
\label{sec:intro}

Generating 3D and 4D scenes from images or text prompts has attracted significant attention due to its potential applications in AR/VR and robotics~\cite{tewari2022advances, chen2024survey, wu2024recent}. This progress is largely driven by the advancement of generative models~\cite{kong2024hunyuanvideo, genmo2024mochi, HaCohen2024LTXVideo, blattmann2023stable, yang2024cogvideox} and neural scene representations~\cite{mildenhall2021nerf, kerbl20233d, wu20244d}. Generative models learn the underlying distribution of large-scale and high-quality video data, leveraging their scalability without imposing explicit physical constraints. In contrast, neural scene representations lift these distributions into structured 3D or 4D spaces, enforcing physical consistency and enabling more faithful scene modeling.

In this context, recent efforts~\cite{he2024cameractrl, wang2024motionctrl, bahmani2024ac3d, bahmani2024vd3d, zhao2024genxd, sun2024dimensionx, gao2024cat3d, wu2024cat4d} have focused on producing geometrically consistent images by fine-tuning generative models with camera pose parameters, where the generation process follows user-defined or pre-defined camera trajectories. This facilitates seamless 3D and 4D scene reconstructions but requires complex optimization to regress neural scene representations, increasing computational overhead and hindering practical adaptation.


To address this inefficiency, another line of work~\cite{li2025director3d, go2024splatflow} has introduced text-conditioned camera pose generation and feed-forward decoding of pixel-aligned 3DGS~\cite{kerbl20233d}. This approach trains 3DGS decoders to learn a mapping function that directly reconstructs 3D scenes from multi-view images, leveraging multiple 3D scene datasets for improved 3D scene modeling. However, the limited datasets used to train pose-conditioned decoders fail to capture the various camera trajectories produced by video generative models. This results in slight misalignments between video generation and scene reconstruction, requiring further refinement steps~\cite{poole2022dreamfusion, zhu2023hifa, yang2023learn} to enhance geometric consistency.



To sum up, previous works handle geometric alignment separately in either video generation or scene reconstruction. This makes it difficult to address cross-stage misalignments, as inconsistencies in one stage may not be fully corrected in the other. Despite recent efforts to mitigate this issue, achieving precise alignment remains an ongoing challenge due to the indistinct link between the two stages.

On the other hand, zero-shot guidance methods that alter the sampling trajectory of generative models to enforce physical constraints have been widely explored in recent literature~\cite{daras2024survey, uehara2025reward}.
These methods have shown great generalizability and performance, contingent upon a {\em well-defined reward}. For instance, in inverse imaging problems, a closed-form likelihood function can guide prior sampling processes toward posterior sampling~\cite{chung2022diffusion, song2023pseudoinverse}. However, directly applying zero-shot guidance to 3D/4D scene generation remains challenging due to the ambiguity in defining reward functions and the significant computational overhead.

In this work, we address this gap by introducing SteerX, a zero-shot inference-time steering method that seamlessly integrates video generation and scene reconstruction, generating geometrically aligned high-quality 3D and 4D scenes. To achieve this, we define geometric reward functions to assess physical consistency across video frames, drawing inspiration from cycle consistency~\cite{zhu2017unpaired} to guide the generation process toward high-reward outputs.
We propose two geometric reward functions tailored for 3D and 4D scene generation. To evaluate geometric consistency across multiple video frames, we extend MEt3R~\cite{asim2025met3r}, a recent evaluation metric for image pairs, by incorporating advanced pose-free feed-forward scene reconstruction methods such as MV-DUSt3R+~\cite{tang2024mv} and MonST3R~\cite{zhang2024monst3r}. These reconstruction methods lift intermediate generated video frames\footnote{Denoised predictions, as used in \cite{chung2022diffusion}} during the reverse sampling process into 3D and 4D spaces. The reconstructed scenes are then projected back into the original image space for consistency evaluation.

However, even when the reward is defined, one cannot directly adopt standard gradient guidance approaches~\cite{chung2022diffusion,song2023pseudoinverse,chung2023decomposed} for several reasons. For one, gradient guidance can be used only when the reward is fully differentiable, which hampers flexibility in the design of reward functions. Moreover, there exist substantial memory constraints to compute gradients for dozens of video frames, limiting their scalability for long video sequences. Therefore, we propose a steering algorithm based on sequential Monte Carlo (SMC)~\cite{doucet2001introduction}, which has gained recent traction~\cite{wu2023practical, li2024derivative, dou2024diffusion,singhal2025general} due to its applicability to non-differentiable reward functions and favorable inference-time scaling properties.

By designing tailored reward functions for geometrically plausible and accurate 3D/4D scene generation, along with a guided sampling process based on SMC, SteerX serves as a fully general framework that integrates {\em any} generative video models with {\em any} 3D reconstruction models, enabling diverse tasks, including Image-to-3D, Image-to-4D, Text-to-3D, and Text-to-4D generation. Through extensive experiments on both 3D and 4D scene generation with various pre-trained video generative models~\cite{go2024splatflow,sun2024dimensionx,genmo2024mochi,kong2024hunyuanvideo,yang2024cogvideox}, we demonstrate the effectiveness and broad applicability of our approach. Furthermore, we show that by increasing the number of particles, we see favorable scaling properties, opening up new possibilities that have been rather untouched in 3D/4D scene generation: test-time scaling.



%% file: sec/2_related_work.tex
\section{Related Works}
\label{sec:related_works}

\subsection{3D and 4D Scene Reconstruction}

Various neural scene representations have been explored for 3D and 4D scene reconstruction. Most recent works tend to use Neural Radiance Fields (NeRF)~\cite{mildenhall2021nerf} and 3D Gaussian Splats (3DGS)~\cite{kerbl20233d} for its high-quality rendering of underlying static scenes~\cite{barron2021mip, barron2023zip, yu2024mip}. These scene representations have been extended to reconstruct 4D scenes by incorporating additional supervisions (\eg, depth maps~\cite{yang2024depth} and segmentation masks~\cite{kirillov2023segment}) along with deformable fields~\cite{park2024point, park2021nerfies, li2021neural, park2021hypernerf, wang2024shape} and Gaussian splats~\cite{wu20244d, wan2024superpoint}. 

Recent advancements in 3D and 4D scene datasets~\cite{yu2023mvimgnet, ling2024dl3dv, reizenstein2021common, zhou2018stereo, liu2021infinite, dai2017scannet, yeshwanth2023scannet++} have enabled the development of generalizable scene representations, allowing feed-forward scene reconstruction methods~\cite{cong2023enhancing, wang2025freesplat, charatan2024pixelsplat, chen2025mvsplat, szymanowicz2024splatter}. Further advancements have extended these techniques to pose-free feed-forward 3D reconstruction models~\cite{wang2024dust3r, ye2024no, tang2024mv, smart2024splatt3r, zhang2024monst3r}, eliminating the need for camera pose conditions and enabling the scene reconstruction from hundreds of unposed images, making the process more flexible and scalable. We leverage these powerful pose-free 3D reconstruction models, specifically MV-DUSt3R+~\cite{tang2024mv} and MonST3R~\cite{zhang2024monst3r}, to assess geometric consistency and lift generated video frames into 3D and 4D spaces.

\subsection{3D and 4D Scene Generation}

Recent 3D and 4D scene generation approaches commonly adopt a two-stage framework, where video generation is followed by scene reconstruction, with a focus on enhancing geometric consistency across both stages. Some works fine-tune video generative models with camera pose parameters~\cite{bahmani2024ac3d, he2024cameractrl, wang2024motionctrl, gao2024cat3d, wu2024cat4d, zhao2024genxd, yu2024viewcrafter, liu2024reconx}, where the generation process is conditioned on a user-defined or pre-defined camera trajectory to produce geometrically aligned video frames. These camera-conditioned video frames are used as training data to regress neural scene representations. While they produce plausible scenes, optimizing scene representations from scratch incurs high computational costs.

In contrast, camera-free approaches~\cite{li2025director3d, go2024splatflow} reconstruct 3D scenes without camera conditions, improving geometric alignment in the scene reconstruction stage. They generate text-conditioned camera poses and use a feed-forward 3DGS decoder trained on diverse 3D scene datasets to estimate underlying 3D structures from posed multi-view images. However, these methods often fail to fully capture the diverse camera trajectories produced by video generative models. Instead of addressing alignment separately at each stage, SteerX integrates scene reconstruction models directly into the video generation process, ensuring that video frames are optimally structured for accurate scene reconstruction and enabling seamless integration of both stages.

\subsection{Guided Sampling}

Generative models are trained to represent data distributions across various domains, including images~\cite{ho2020denoising, peebles2023scalable, park2024switch, park2023denoising}, videos~\cite{ho2022video, harvey2022flexible},  and 3D~\cite{liu2023syncdreamer, woo2024harmonyview}. 
However, sampling from tilted distributions (\eg, posterior distributions from measurements) is often preferable to drawing from naive prior.


Recent works~\cite{xu2023imagereward, black2024training} have encoded user preferences through reward functions and fine-tuned generative models to maximize these rewards, effectively sampling from a tilted distribution toward high-reward outputs. However, fine-tuning approaches require extensive training and hampers the generalizability by modifying the base distribution.

In contrast, guided sampling methods~\cite{song2020score,chung2022diffusion,song2023pseudoinverse,daras2024survey,uehara2025reward} allow sampling from a tilted distribution in a zero-shot manner while keeping the base distribution intact. Among various paradigms, two classes stand out: gradient-based guidance~\cite{chung2022diffusion,song2023pseudoinverse} and particle \add{filtering}\footnote{Throughout the paper, we flexibly use the term particle filtering and SMC to refer to this class.}\cite{wu2023practical, li2024derivative, dou2024diffusion}. Gradient-based guidance is simple to implement and widely applied but requires differentiable reward functions and is memory-intensive. In contrast, particle filtering is amenable to non-differentiable rewards, and is more memory-efficient, as it does not rely on backpropagation.
Recently, Feynman-Kac Steering (FKS)~\cite{singhal2025general} has unified previous \add{particle filtering} techniques, introducing a general steering framework that can be applied to any reward function. SteerX builds upon FKS, integrating geometric reward functions and exploring a tailored design space for geometric steering, making it applicable to {\em any} video generative models.

%% file: sec/3_preliminary.tex
\input{figure/overview}

\section{Preliminary: Feynman-Kac Steering}
\label{sec:prelim}

Diffusion models~\cite{song2020score, dhariwal2021diffusion} are learned to reverse a forward process $\mathbf{x}_t \sim q(\mathbf{x}_t | \mathbf{x}_0 = \mathbf{x})$ with $T$ timesteps, that gradually transforms a data sample $\mathbf{x}$ into Gaussian noise. This reverse process $p_{\theta}(\mathbf{x}_{t-1} | \mathbf{x}_t)$ is defined as:
\begin{equation}
\label{eq:reverse_generative_dist}
    p_{\theta}(\mathbf{x}_{T}, \mathbf{x}_{T-1}, \dots, \mathbf{x}_1, \mathbf{x}_{0}) \propto \prod_{t=1}^{T}{p_{\theta}(\mathbf{x}_{t-1} | \mathbf{x}_t)}.
\end{equation}

The main goal of inference-time steering methods is to sample tilted distributions toward maximizing an exponential of user-specified reward function $r_{\phi}(\mathbf{x}_0)$ as:
\begin{equation}
\label{eq:tilted_dist}
    \tilde{p}_{\theta}(\mathbf{x}_{0}) \propto {p_{\theta}(\mathbf{x}_0)\exp(\lambda r_{\phi}(\mathbf{x}_0))}.
\end{equation}

Feynman-Kac Steering (FKS)~\cite{singhal2025general} introduces a system of multiple interacting diffusion paths (\ie. particles), where each path $\mathbf{x}_{T:0}=(\mathbf{x}_T, \dots, \mathbf{x}_0)$ is resampled at intermediate steps based on potentials $G_t(\mathbf{x}_{T:t})$, resulting in a sequence of tilted distributions $\tilde{p}_{\theta}(\mathbf{x}_{T:t})$:
\begin{equation}
    \tilde{p}_{\theta}(\mathbf{x}_{T:t}) \propto \prod_{s=T}^{t}{p_{\theta}(\mathbf{x}_{T:s})G_s(\mathbf{x}_{T:s})},
\end{equation}
where the product of potentials matches the total tilt of the base distribution as  $\prod_{t=T}^{0}{G_t(\mathbf{x}_{T:t})}=\exp(\lambda r_{\phi}(\mathbf{x}_0))$. These potentials are generally defined as SMC~\cite{wu2023practical} to score the particles at each transition step $\mathbf{x}_{t+1} \rightarrow \mathbf{x}_t$ based on intermediate rewards $r_{\phi}(\mathbf{x}_t)$. 


%% file: figure/overview.tex
\begin{figure*}[t]
    \centering
    \includegraphics[width=\linewidth]{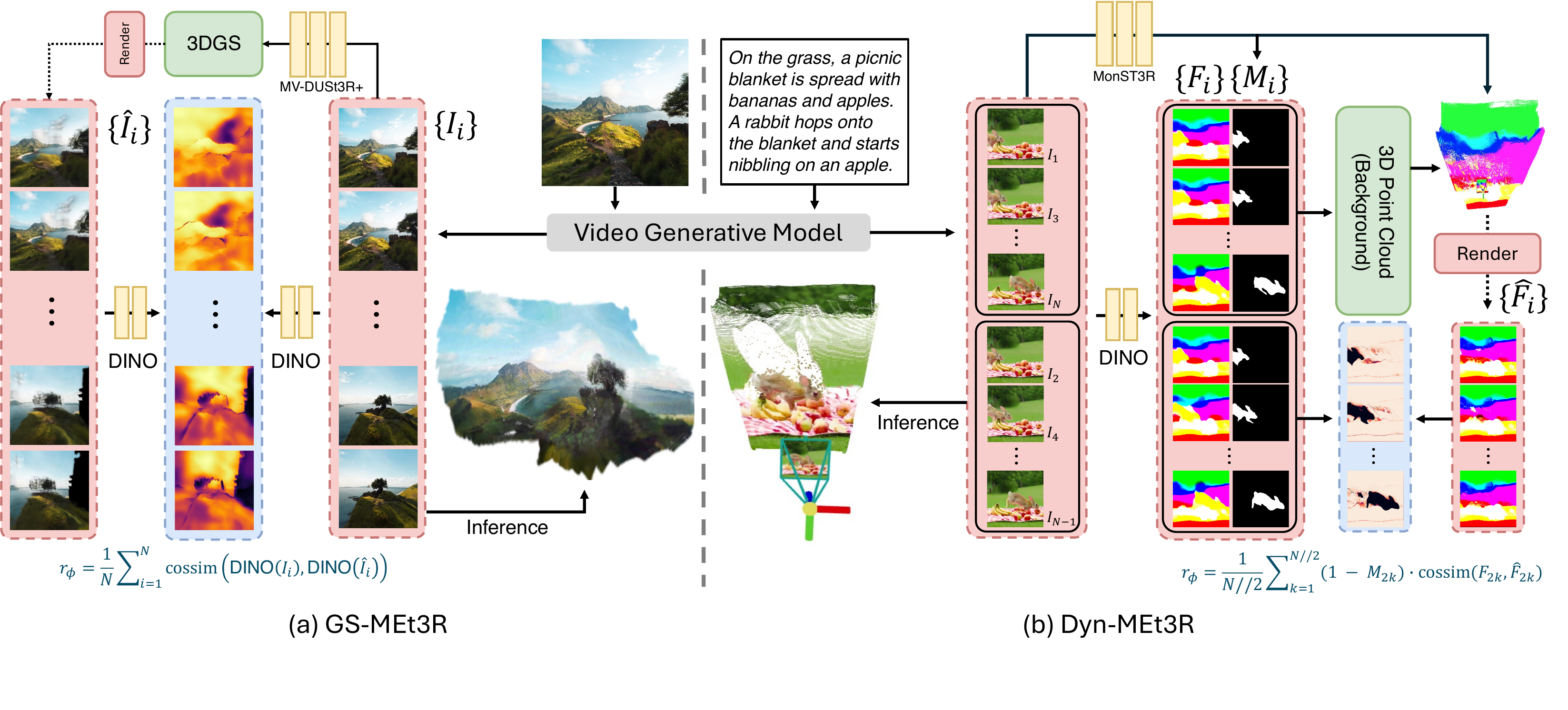}
    \vspace{-1.3cm}
    \caption{\textbf{An overview of geometric rewards.} Our reward functions assess the geometric consistency of intermediate generated video frames by computing the feature similarity of upscaled DINO features. (a) GS-MEt3R evaluates feature similarity between the original video frames and their corresponding rendered images from 3DGS. (b) Dyn-MEt3R focuses on background regions by unprojecting background features from half of the video frames and reprojecting them onto the remaining frames to compute feature similarity.}
    \vspace{\belowfigcapmargin}
    \label{fig:overview}
\end{figure*}

%% file: sec/4_method.tex
\section{Methods}
\label{sec:methods}

In this section, we introduce SteerX, a zero-shot inference-time steering method for 3D/4D scene generation. SteerX unifies pose-free feed-forward scene reconstruction models into the video generation process, iteratively tilting the data distribution towards geometrically aligned samples. In~\Cref{sec:rewards}, we define two geometric reward functions to evaluate geometric consistency in generated multi-view images and dynamic videos, respectively. In~\Cref{sec:steer}, we detail our SMC-based steering algorithm.

\subsection{Geometric Rewards}
\label{sec:rewards}

Our geometric rewards build upon MEt3R~\cite{asim2025met3r}, which utilizes DUSt3R~\cite{wang2024dust3r} to measure feature similarity in overlapping regions between image pairs. It extracts upscaled DINO features~\cite{caron2021emerging, fu2024featup} and computes the cosine similarity between reference view features and the rendered features from another viewpoint. It can evaluate multiple images by computing feature similarity across all possible image pairs and averaging the scores. However, as the number of images increases, this becomes computationally expensive, making the overall generation process inefficient.

To address this, as shown in~\cref{fig:overview}, we introduce two geometric reward functions, GS-MEt3R and Dyn-MEt3R, which measure geometric consistency across video frames tailored for 3D/4D scene generation, respectively. We employ pose-free feed-forward scene reconstruction methods, MV-DUSt3R+~\cite{tang2024mv} and MonST3R~\cite{zhang2024monst3r}, to reconstruct 3D and 4D scenes, which are subsequently projected back into image space to assess consistency with the original frames. This mitigates computational bottlenecks while efficiently evaluating global consistency in the generated video frames.

\vspace{\paramargin}
\paragraph{3D Scene Reward.}
We extend MEt3R~\cite{asim2025met3r} in distinct ways to better suit 3D and 4D scene reconstruction models. For 3D scenes, recent methods~\cite{ye2024no, tang2024mv} introduce a mapping function $f_\phi$ that directly reconstructs 3DGS and camera poses from $N$ sparse views $\{I_i \in \mathbb{R}^{H \times W \times 3}\}^{N}_{i=1}$ as:
\begin{equation}
    f_{\phi} : \{I_i\}_{i=1}^{N} \rightarrow 
    \left\{
    \begin{array}{@{}l@{}}
        \{ \boldsymbol{\mu}_j, \mathbf{o}_j, \mathbf{\Sigma}_j, \mathbf{c}_j \}_{j=1}^{N \times H \times W} \\[6pt]
        \{P_i\}_{i=1}^{N}
    \end{array}
    \right.,
\end{equation}
where the 3D scene is represented as Gaussian parameters, including position $\boldsymbol{\mu}$, volume density $\mathbf{o}$, covariance $\mathbf{\Sigma}$, and color $\mathbf{c}$. Then, we produce images $\{\hat{I}_i\}^N_{i=1}$ by rendering the scene with estimated camera poses $\{P_i\}^N_{i=1}$, ensuring they correspond to the same viewpoints as the input images. Finally, GS-MEt3R is measured by computing the cosine similarity between upscaled DINO~\cite{caron2021emerging, fu2024featup} features of input images $\{F_i\}^{N}_{i=1}$ and rendered images $\{\hat{F}_i\}^{N}_{i=1}$ as:
\begin{equation}
    r_\phi = \frac{1}{N}\sum\limits_{i=1}^{N}\sum\limits_{j=1}^{H}\sum\limits_{k=1}^{W} \frac{F^{jk}_i \cdot \hat{F}^{jk}_i }{\vert\vert F^{jk}_i \vert\vert \vert\vert \hat{F}^{jk}_i \vert\vert}.
\end{equation}

This approach not only offers a more direct evaluation of the physical consistency of the 3D scene compared to averaging consistency across all image pair combinations but also indirectly assesses the rendering quality of 3DGS. In other words, high GS-MEt3R scores indicate both geometric alignment and visually realistic 3D scenes.

\vspace{\paramargin}
\paragraph{4D Scene Reward.}
While the 3D scene reward function is based on 3DGS, feed-forward dynamic scene reconstruction with Gaussians remains underexplored, making it difficult to directly apply 3DGS-based rewards. Instead, we employ 3D point cloud representations, where MonST3R~\cite{zhang2024monst3r} reconstructs it with point maps $\{X_i\}_{i=1}^{N}$, binary dynamic masks $\{M_i\}_{i=1}^{N}$, and camera poses $\{P_i\}_{i=1}^{N}$ as:
\begin{equation}
    f_{\phi} : \{I_i\}_{i=1}^{N} \rightarrow \{X_i, M_i, P_i\}_{i=1}^{N},
\end{equation}
where we leverage these time-varying point clouds as 4D scene representations and design a reward function for evaluating geometric consistency in dynamic videos. 

Since dynamic masks are produced in the camera pose estimation process to retain only high-confidence points, a well-reconstructed 4D scene should effectively filter out dynamic objects while preserving geometric consistency in the background regions. Therefore, we evaluate the consistency only for background regions of video frames, which are not filtered out by the dynamic mask. To this end, we first split $N$ video frames into two subsets: $\mathcal{I}_{src} = \{I_1, I_3, \dotsc, I_N\}$ and $\mathcal{I}_{tgt} = \{I_2, I_4, \dotsc, I_{N-1}\}$. Then, we unproject the upsampled DINO features of background regions in $\mathcal{I}_{src}$ into 3D space using MonST3R. Finally, we reproject them onto the viewpoint of $\mathcal{I}_{tgt}$, where the rendered features $\hat{\mathcal{F}}_{tgt} = \{\hat{F}_1, \hat{F}_3, \dotsc, \hat{F}_N\}$ are used to compute the feature similarity with background regions in $\mathcal{I}_{tgt}$ as:
\begin{equation}
    r_i =\sum\limits_{j=1}^{H}\sum\limits_{k=1}^{W} (1 -M^{jk}_{i})\frac{F^{jk}_{i} \cdot \hat{F}^{jk}_{i} }{\vert\vert F^{jk}_{i} \vert\vert \vert\vert \hat{F}^{jk}_{i} \vert\vert},
\end{equation}%
\begin{equation}
    r_\phi = \frac{1}{(N//2)}\sum\limits_{i=1}^{N//2}r_i.
\end{equation}

With these geometric rewards, we can effectively steer pre-trained video generative models to produce geometrically consistent video frames, which are then directly reconstructed into 3D and 4D spaces using the feed-forward 3D reconstruction models employed in the geometric rewards.

\subsection{Geometric Steering}
\label{sec:steer}

\input{algo/fkds}
Using the rewards defined in \Cref{sec:rewards}, let $\tildep$ in \eqref{eq:tilted_dist} be the distribution we wish to sample from. SMC operates with the three following steps:
\begin{enumerate}[itemsep=2mm]
    \item (\textbf{Proposal}) For each particle $i$, sample from the proposal distribution $\x_t^i \sim q_t(\x_t|\x_{t+1}^i)$
    \item (\textbf{Weighting}) Compute weights from reward-based potentials $\omega_t^i = \frac{p_\theta(\x_t^i|\x_{t+1}^i)}{q_t(\x_t^i|\x_{t+1}^i)}G_t(\x_{T:t}^i)$
    \item (\textbf{Resampling}) Draw new particles from the multinomial distribution $\{\mathbf{x}^{i}_t\}^k_{i=1} \sim \text{Multinomial}(\{\mathbf{x}^i_{t}, G^i_{t}\}^k_{i=1})$
\end{enumerate}
Two choices should be made: the potential $G_t$, and the proposal distribution $q_t$. For the potential, we use max potential
\begin{align}
\label{eq:max_potential}
    G_t(\x_{T:t})^i := \exp \left(
    \lambda \max_{j=t}^T [r_\phi(\hat\x_0)]
    \right),
\end{align}
with
\begin{align}
\label{eq:max_potential_0}
    G_0(\x_{T:0}) := \exp \left(\lambda r_\phi(\x_0)\right)\left(
    \prod_{t=1}^T G_t(\x_{T:t})
    \right)^{-1},
\end{align}
such that the particle with the highest reward is preferred. Notice that we use the Tweedie estimate $\hat\x_0 = \mathbb{E}[\x_0|\x_t]$~\cite{efron2011tweedie,chung2022diffusion,singhal2025general} in intermediate steps to avoid full reverse sampling.
For the proposal kernel, to save computation, we leverage DPM-solver++~\cite{lu2022dpm}, which approximates the true sampling trajectory limited to small neural function evaluation (NFE). These choices lead to SteerX, as shown in Alg.~\ref{alg:geo_steering}.

\begin{restatable}{proposition}{convergence}
\label{prop:convergence}
Given the reverse generative process in \eqref{eq:reverse_generative_dist}, let $q_t$ be the transition kernel satisfying
\begin{align}
    \frac{p_\theta(\x_{t-1}|\x_t)}{q_t(\x_{t-1}|\x_t)} = 1 + \epsilon_t(\x_{T:t}),
\end{align}
with $|\epsilon_t(\x_{T:t})| \leq \varepsilon$ uniformly. 
Also assume that the error from the reward computed at the approximate state $\hat\x_0$ is bounded, i.e. $|r_\phi(\hat\x_0) - r_\phi(\x_0)| \leq \eta$.
Then, given the defined max potentials in \eqref{eq:max_potential},\eqref{eq:max_potential_0}, Alg.~\ref{alg:geo_steering} samples from
\begin{align}
    \tildep(\x_0) \propto p_\theta(\x_0) \exp (\lambda r_\phi(\x_0))(1 + \mathcal{O}(T\varepsilon + \lambda\eta))
\end{align}
\end{restatable}
See Appendix~\ref{sec:proof} for the proof. Prop.~\ref{prop:convergence} states that when the approximation errors are sufficiently small, then we can sample from the desired tilted distribution with SteerX.

%% file: algo/fkds.tex
\begin{figure}[t]
\vspace{-1em}
\begin{algorithm}[H]
    \textbf{Required:} $\mathbf{v}$-parametrized diffusion model $\mathbf{v}_\theta$, reward function $r_\phi$, number of particles $k$, and initial noise $\{\mathbf{x}^{i}_{T}\}^{k}_{i=1} \sim \mathcal{N}(0, I)$.
    \caption{SteerX (v-prediction)}\label{alg:geo_steering} 
    \textbf{Sampling:}
    \begin{algorithmic}[1]
        \For{$t \in \{T - 1, \dotsc, 0\}$}
            \vspace{1mm}
            \For{$i \in \{1 \dotsc k\}$}
                \vspace{1mm}
                \State $\hat{\mathbf{x}}^{i}_{0} \gets \sqrt{\bar{\alpha}_{t+1}}\mathbf{x}^{i}_{t+1} - \sqrt{1 - \bar{\alpha}_{t+1}}\mathbf{v}_\theta(\mathbf{x}^{i}_{t+1})$
                \vspace{1mm}
                \State $\mathbf{x}^{i}_{t} \gets \textit{dpm-solver}(\hat{\mathbf{x}}^{i}_{0}, \mathbf{x}^{i}_{t+1})$
                \vspace{1mm}
                \State $\mathbf{s}^{i}_{t} \gets r_\phi(\hat{\mathbf{x}}^{i}_{0}) \hfill \triangleright \textit{Intermediate rewards}$
                \vspace{1mm}
                \State ${G}^{i}_{t} \gets \exp(\lambda\max^T_{j=t}(\mathbf{s}^{i}_{j})) \hfill \triangleright \textit{Potential}$
                \vspace{1mm}
            \EndFor
            \vspace{1mm}
            \State $\{\mathbf{x}^{i}_t\}^k_{i=1} \sim \text{Multinomial}(\{\mathbf{x}^i_{t}, G^i_{t}\}^k_{i=1}$) \hfill $\triangleright$ \textit{Resample}
        \vspace{1mm}
        \EndFor
        \vspace{1mm}
        \State $l \gets \argmax_{i \in \{1, \dotsc, k\}} r_\phi(\mathbf{x}^i_0)$
        \State \textbf{return} $\mathbf{x}^l_0$
    \end{algorithmic}
\end{algorithm}
\vspace{-2.2em}
\end{figure}

%% file: sec/5_experiments.tex
\section{Experimental Results}
\label{sec:exp}

In this section, we conduct extensive experiments to verify the scalability and effectiveness of SteerX across various video generative models in four scene generation scenarios: Text-to-4D, Image-to-4D, Text-to-3D, and Image-to-3D. We briefly explain our experimental setup in~\Cref{sec:exp_setup} and present both qualitative and quantitative results, as well as design choices of SteerX in~\Cref{sec:exp_results}.

\input{figure/text-to-4d}

\subsection{Experimental Setup}
\label{sec:exp_setup}

\paragraph{Implementation Details.}


For video generative models, we generate 25 frames at a $480 \times 480$ resolution, except for models that require a fixed video length. Since reconstructing 3D scenes using all generated video frames is impractical, we use only eight frames for the reconstruction. We follow the default settings of MV-DUST3R+~\cite{tang2024mv} and MonST3R~\cite{zhang2024monst3r}. We set $\lambda = 10$ for the potential. To encourage exploration and reduce computation overhead, we adopt interval resampling, as in FKS~\cite{singhal2025general}, and apply the steering procedure at $M$ selected timesteps out of $T$ reverse diffusion steps, setting $M = 4$ for 3D scene generation and $M = 2$ for 4D scene generation.

\vspace{\aboveeqmargin}
\paragraph{Baselines and Benchmarks.}

We evaluate our SteerX with Mochi~\cite{genmo2024mochi} and HunyuanVideo~\cite{kong2024hunyuanvideo} for Text-to-4D generation, and CogVideoX~\cite{yang2024cogvideox} for Image-to-4D. For Image-to-3D, we employ S-Director with orbit-left from DimensionX~\cite{sun2024dimensionx}. For Text-to-3D, we utilize the video generation and scene reconstruction models proposed in SplatFlow~\cite{go2024splatflow}. We conduct evaluations on 98 samples from the PenguinVideoBenchmark~\cite{kong2024hunyuanvideo} with camera descriptions for Text-to-4D generation, and 355 samples from VBench-I2V~\cite{huang2024vbench} for Image-to-3D and Image-to-4D generation. Additionally, we use 100 samples from the Single-Object-with-Surrounding set of T3Bench~\cite{he2023t3bench} for Text-to-3D evaluation. For all evaluations, we compare against the best-of-N approach, where $k$ particles are generated independently, and the one with the highest reward is selected. Additionally, we include a baseline ($k=1$) that generates a video and reconstructs the scene without inference-time steering.

\vspace{\aboveeqmargin}
\paragraph{Evaluation Metrics.}
For 4D generation, we evaluated the Aesthetic Score and Subject Consistency in VBench~\cite{huang2024vbench} for video quality and semantic consistency, respectively. Also, we evaluated Temporal Consistency to measure text alignment in Text-to-4D generation, including adherence to camera motion instructions. We evaluated Dynamic Degree to measure the overall dynamicness in Image-to-4D generation. For 3D generation, we followed the evaluation protocols in Director3D~\cite{li2025director3d} and SplatFlow~\cite{go2024splatflow}, measuring the image quality and the CLIP score~\cite{hessel2021clipscore} for both generated multi-view images and rendered images of 3DGS. Finally, for all generation scenarios, we evaluated the geometric consistency of multi-view images and dynamic videos using GS-MEt3R and Dyn-MEt3R, respectively. 

\input{table/penguin}

\subsection{Main Results}
\label{sec:exp_results}

\paragraph{Text-to-4D Generation.}

\Cref{fig:mochi} and \Cref{tab:penguin} present results of text-conditioned video generation. Notably, SteerX significantly outperforms both the baseline and the best-of-N, confirming that our geometric steering effectively resamples particles to guide the data distribution toward geometrically aligned samples. Moreover, quantitative results show a strong correlation between Dyn-MEt3R and other metrics, suggesting that the geometric steering can contribute to both physical consistency and overall video quality.

We observe that annealed resampling timesteps, where the steering procedure at the early sampling stages outperforms linear resampling, whereas resampling at later stages tends to degrade performance. This observation aligns with previous findings~\cite{bahmani2024ac3d, sun2024dimensionx}, where camera poses for video frames are largely determined early in the sampling process as the diffusion model conceptualizes global structures. As a result, resampling at early timesteps coarsely aligns the geometry by tilting the data distribution, laying the foundation for a more precise geometric alignment in later stages.


\input{figure/cog}
\input{figure/image-to-4d}

\input{figure/qual_4d}

The geometrically aligned video frames contribute to better 4D scene reconstruction, as shown in~\cref{fig:qual_4d}. While the baseline struggles to capture accurate camera poses and the best-of-N approach results in blurry 4D scenes, our approach precisely estimates camera poses and generates visually realistic 4D scenes. This highlights the effectiveness of SteerX in the seamless integration of video generation and scene reconstruction, ensuring that generated video frames are optimally structured for precise scene reconstruction.

\vspace{\aboveeqmargin}
\paragraph{Image-to-4D Generation.}
\Cref{fig:cog}~and~\Cref{tab:vbench_4d} show the results in image-conditioned video generation, and our SteerX performs the best for all metrics, and we observe that SteerX generates dynamic and realistic video frames, whereas the baseline produces almost motionless frames, and the best-of-N approach results in blurry frames. Also, we validate the scalability of SteerX, as increasing the number of particles $k$ consistently enhances performance across all metrics. As shown in~\cref{fig:image_to_4d}, geometrically aligned video frames can be reconstructed to better 4D scenes, producing realistic dynamic objects and well-aligned backgrounds.

\input{table/vbench_4d}

\input{table/vbench_3d}

\input{figure/splatflow}

\input{figure/dimensionx}

\vspace{\aboveeqmargin}
\paragraph{Image-to-3D Generation.}
We verify the effectiveness of SteerX in the Image-to-3D scene generation. While multi-view images generated by SteerX closely resemble those from the best-of-N approach, SteerX reconstructs 3D scenes with finer details and fewer discontinuities in novel views, as shown in~\cref{fig:dimensionx}. Furthermore, quantitative results in~\Cref{tab:vbench_3d} demonstrate that SteerX achieves better visual quality and geometric consistency in both rendered images and generated images compared to other baselines.

\input{table/t3bench}

\input{table/t3bench_render}

\vspace{\aboveeqmargin}
\paragraph{Text-to-3D Generation.}
Notice that SteerX can be applied to any generative model and scene reconstruction model, including the multi-view rectified flow model and the feed-forward 3DGS decoder in SplatFlow~\cite{go2024splatflow}, enabling geometric steering through GS-MEt3R rewards. \Cref{tab:t3bench} presents quantitative results of generated multi-view images, demonstrating that SteerX effectively enhances overall performance and further improves scalability as the number of particles increases. Furthermore, as shown in~\Cref{tab:t3bench_render}, these 3D-aligned multi-view images are integrated with existing refinement processes, achieving state-of-the-art performance in 3D scene generation. \Cref{fig:splatflow} presents qualitative results, demonstrating that SteerX produces more text-aligned 3D scenes and visually realistic rendered images.

%% file: figure/text-to-4d.tex
\begin{figure*}[t!]
    \centering
    \setlength\tabcolsep{2pt}
    
    \begin{tabular}{c:cc:cc:cc:cc}
        \multirow{2}[2]{*}[0.06\linewidth]{%
            \parbox{0.12\linewidth}{%
                \centering
                \footnotesize
                \textit{The butterfly lands on the railing, and \textcolor{red}{the camera pans to the right} to capture the tulip field next to it. The entire video is presented in anime style.}%
                    }%
        }  &
        \adjincludegraphics[clip,width=0.1\linewidth,trim={0 0 0 0}]{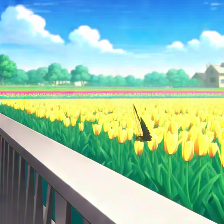} &
        \adjincludegraphics[clip,width=0.1\linewidth,trim={0 0 0 0}]{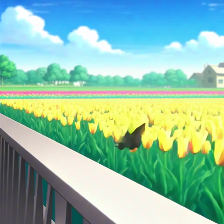} &
        \adjincludegraphics[clip,width=0.1\linewidth,trim={0 0 0 0}]{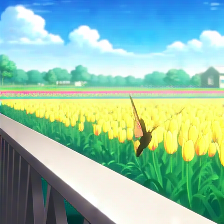} &
        \adjincludegraphics[clip,width=0.1\linewidth,trim={0 0 0 0}]{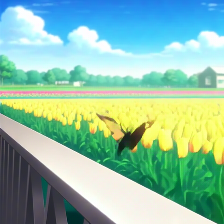} &
        \adjincludegraphics[clip,width=0.1\linewidth,trim={0 0 0 0}]{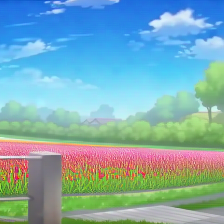} &
        \adjincludegraphics[clip,width=0.1\linewidth,trim={0 0 0 0}]{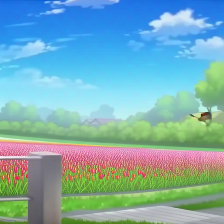} &
        \adjincludegraphics[clip,width=0.1\linewidth,trim={0 0 0 0}]{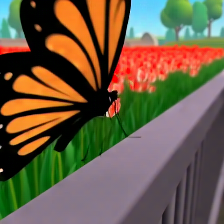} &
        \adjincludegraphics[clip,width=0.1\linewidth,trim={0 0 0 0}]{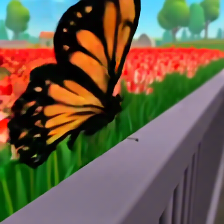} \\
         &
        \adjincludegraphics[clip,width=0.1\linewidth,trim={0 0 0 0}]{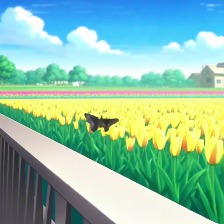} &
        \adjincludegraphics[clip,width=0.1\linewidth,trim={0 0 0 0}]{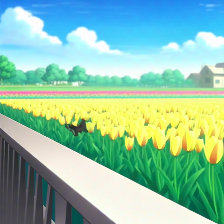} &
        \adjincludegraphics[clip,width=0.1\linewidth,trim={0 0 0 0}]{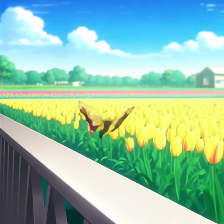} &
        \adjincludegraphics[clip,width=0.1\linewidth,trim={0 0 0 0}]{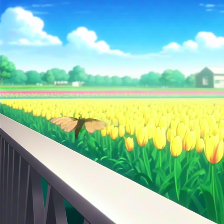} &
        \adjincludegraphics[clip,width=0.1\linewidth,trim={0 0 0 0}]{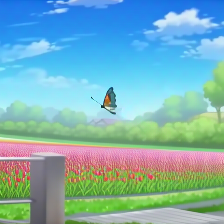} &
        \adjincludegraphics[clip,width=0.1\linewidth,trim={0 0 0 0}]{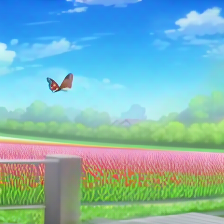} &
        \adjincludegraphics[clip,width=0.1\linewidth,trim={0 0 0 0}]{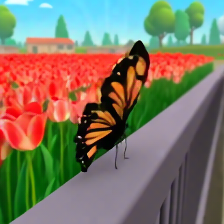} &
        \adjincludegraphics[clip,width=0.1\linewidth,trim={0 0 0 0}]{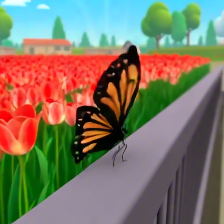} \\
        \arrayrulecolor{gray}\midrule

          \multirow{2}[2]{*}[0.05\linewidth]{%
            \parbox{0.12\linewidth}{%
                \centering
                \footnotesize
                \textit{Silver skyscrapers in the city, with streams of people passing by below. \textcolor{red}{The camera moves vertically from top to bottom during filming.}}%
                    }%
        }  &
        \adjincludegraphics[clip,width=0.1\linewidth,trim={0 0 0 0}]{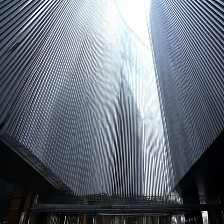} &
        \adjincludegraphics[clip,width=0.1\linewidth,trim={0 0 0 0}]{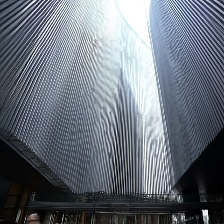} &
        \adjincludegraphics[clip,width=0.1\linewidth,trim={0 0 0 0}]{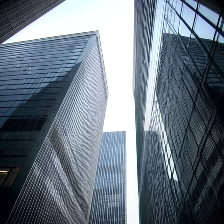} &
        \adjincludegraphics[clip,width=0.1\linewidth,trim={0 0 0 0}]{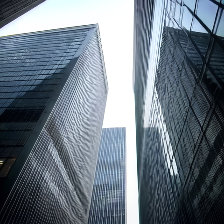} &
        \adjincludegraphics[clip,width=0.1\linewidth,trim={0 0 0 0}]{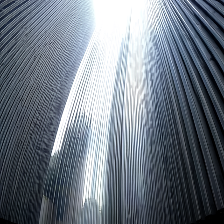} &
        \adjincludegraphics[clip,width=0.1\linewidth,trim={0 0 0 0}]{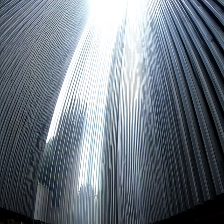} &
        \adjincludegraphics[clip,width=0.1\linewidth,trim={0 0 0 0}]{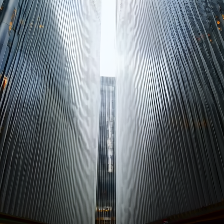} &
        \adjincludegraphics[clip,width=0.1\linewidth,trim={0 0 0 0}]{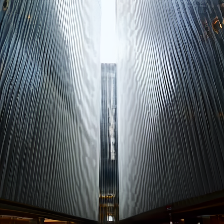} \\
         &
        \adjincludegraphics[clip,width=0.1\linewidth,trim={0 0 0 0}]{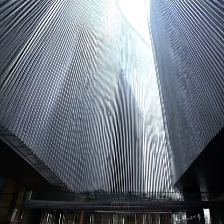} &
        \adjincludegraphics[clip,width=0.1\linewidth,trim={0 0 0 0}]{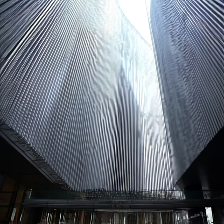} &
        \adjincludegraphics[clip,width=0.1\linewidth,trim={0 0 0 0}]{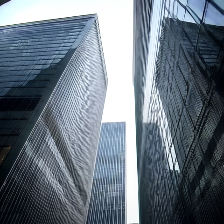} &
        \adjincludegraphics[clip,width=0.1\linewidth,trim={0 0 0 0}]{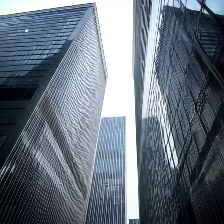} &
        \adjincludegraphics[clip,width=0.1\linewidth,trim={0 0 0 0}]{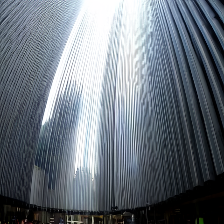} &
        \adjincludegraphics[clip,width=0.1\linewidth,trim={0 0 0 0}]{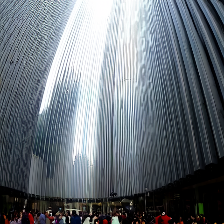} &
        \adjincludegraphics[clip,width=0.1\linewidth,trim={0 0 0 0}]{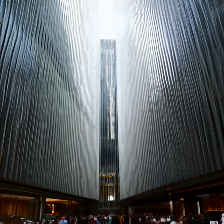} &
        \adjincludegraphics[clip,width=0.1\linewidth,trim={0 0 0 0}]{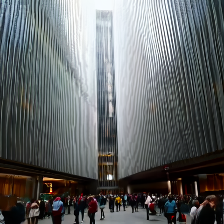} \\

        \multicolumn{1}{c}{\small Text} & \multicolumn{2}{c}{\small Baseline} & \multicolumn{2}{c}{\small + Best-of-N} & \multicolumn{2}{c}{\small + SteerX (later)} & \multicolumn{2}{c}{\small + SteerX (early)}
    \end{tabular}
    \vspace{-1mm}
    \vspace{\abovefigcapmargin}
    \caption{\textbf{Qualitative results of video generation in PenguinVideoBenchmark~\cite{kong2024hunyuanvideo}.} We visualize for Mochi~\cite{genmo2024mochi} (top) and HunyuanVideo~\cite{kong2024hunyuanvideo} (bottom). SteerX achieves the best alignment with camera motion instructions while maintaining high video quality.}
    \vspace{\belowfigcapmargin}
    \label{fig:mochi}
\end{figure*}

%% file: table/penguin.tex
\begin{table}[t]
    \centering
    \setlength\tabcolsep{2pt}
    \resizebox{\linewidth}{!}{
    \begin{tabular}{lcccccc}
       \toprule
        Method & \, $k$ \, & Aesthetic$\uparrow$ & Subject$\uparrow$ & Temporal$\uparrow$ & Dyn-MEt3R$\uparrow$ \\
       \midrule
        Mochi~\cite{genmo2024mochi} & 1 & 0.491 & \underline{0.950} & 0.243 & 0.884 \\
        \arrayrulecolor{gray}\midrule

        \; + BoN & 4 & \underline{0.498} & 0.941 & \underline{0.244} & 0.912 \\
        \rowcolor{gray!25} \; + \textbf{SteerX} (later) & 4 & 0.488 & 0.937 & 0.242 & 0.910 \\
        \rowcolor{gray!25} \; + \textbf{SteerX} (linear) & 4 & 0.490 & 0.949 & \underline{0.244} & \underline{0.918} \\
        \rowcolor{gray!25} \; + \textbf{SteerX} (early) & 4 & \textbf{0.500} & \textbf{0.955} & \textbf{0.248} & \textbf{0.929}\\
        
        \arrayrulecolor{black}\midrule \midrule
        
        HunyuanVideo~\cite{kong2024hunyuanvideo} & 1 & 0.549 & 0.967 & \underline{0.241} & 0.911 \\
        \arrayrulecolor{gray}\midrule

        \; + BoN & 4 & 0.551 & \underline{0.978} & 0.239 & 0.931 \\
        \rowcolor{gray!25} \; + \textbf{SteerX} (later) & 4 & \underline{0.555} & 0.976 & 0.237 & 0.931 \\
        \rowcolor{gray!25} \; + \textbf{SteerX} (linear) & 4 & \textbf{0.556} & 0.973 & \underline{0.241} & \underline{0.943} \\
        \rowcolor{gray!25} \; + \textbf{SteerX} (early) & 4 & \underline{0.555} & \textbf{0.980} & \textbf{0.243} & \textbf{0.964} \\
       \arrayrulecolor{black}\bottomrule
    \end{tabular}
    }
    \vspace{-1mm}
    \vspace{\abovetabcapmargin}
    \caption{\textbf{Quantitative results in PenguinVideoBenchmark~\cite{kong2024hunyuanvideo}.} SteerX with annealed resampling timesteps to the early sampling process significantly improves the overall video quality.}
    \vspace{\belowtabcapmargin}
    \label{tab:penguin}
\end{table}

%% file: figure/cog.tex
\begin{figure*}[t!]
    \centering
    \setlength\tabcolsep{1.1pt}
    
    \begin{tabular}{c:ccc:ccc:ccc}
    

        
        \adjincludegraphics[clip,width=0.095\linewidth,trim={0 0 0 0}]{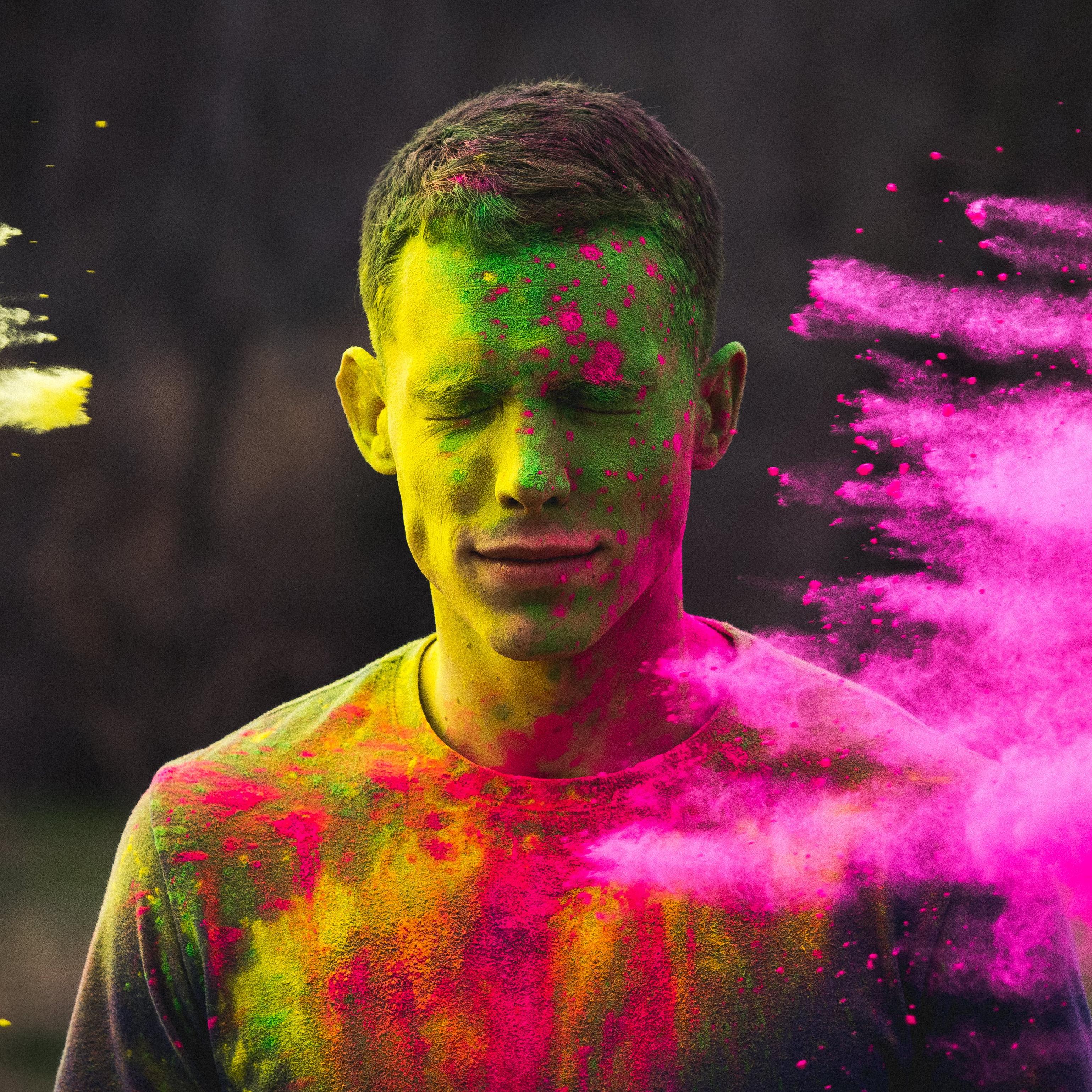} &
        \adjincludegraphics[clip,width=0.095\linewidth,trim={0 0 0 0}]{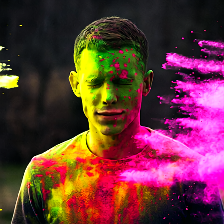} &
        \adjincludegraphics[clip,width=0.095\linewidth,trim={0 0 0 0}]{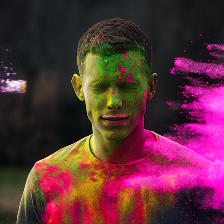} &
        \adjincludegraphics[clip,width=0.095\linewidth,trim={0 0 0 0}]{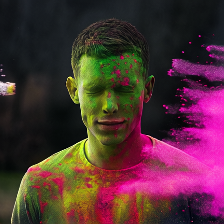} &
        \adjincludegraphics[clip,width=0.095\linewidth,trim={0 0 0 0}]{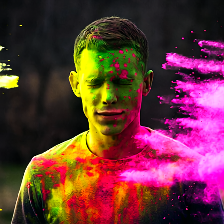} &
        \adjincludegraphics[clip,width=0.095\linewidth,trim={0 0 0 0}]{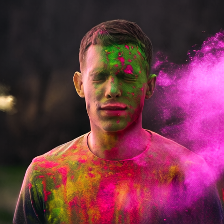} &
        \adjincludegraphics[clip,width=0.095\linewidth,trim={0 0 0 0}]{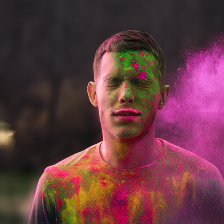} &
        \adjincludegraphics[clip,width=0.095\linewidth,trim={0 0 0 0}]{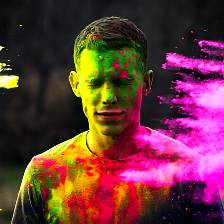} &
        \adjincludegraphics[clip,width=0.095\linewidth,trim={0 0 0 0}]{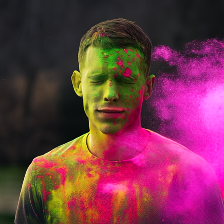} &
        \adjincludegraphics[clip,width=0.095\linewidth,trim={0 0 0 0}]{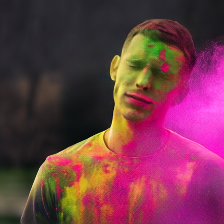} \\

        \adjincludegraphics[clip,width=0.095\linewidth,trim={0 0 0 0}]{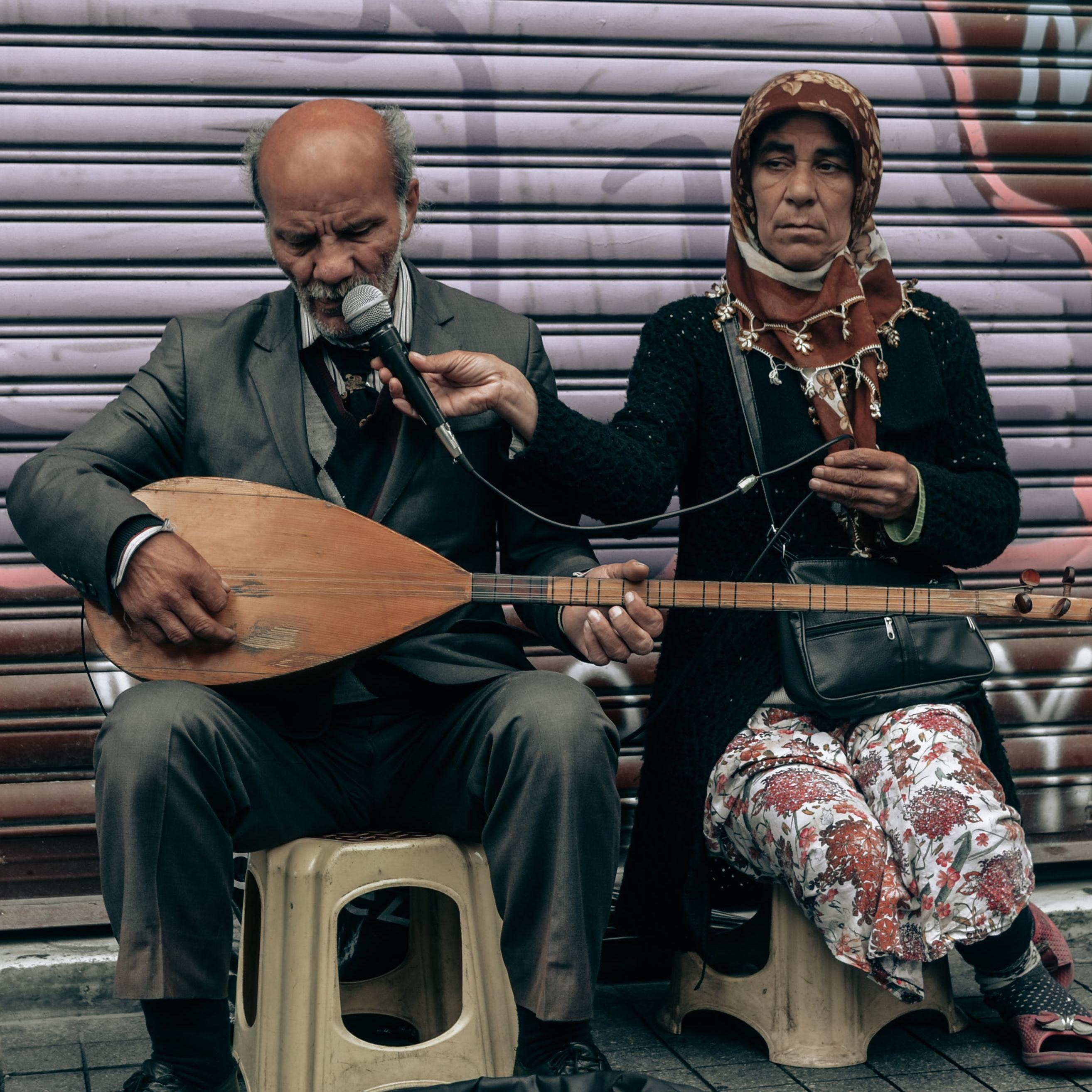} &
        \adjincludegraphics[clip,width=0.095\linewidth,trim={0 0 0 0}]{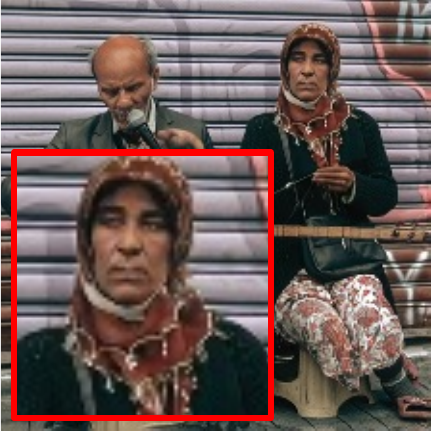} &
        \adjincludegraphics[clip,width=0.095\linewidth,trim={0 0 0 0}]{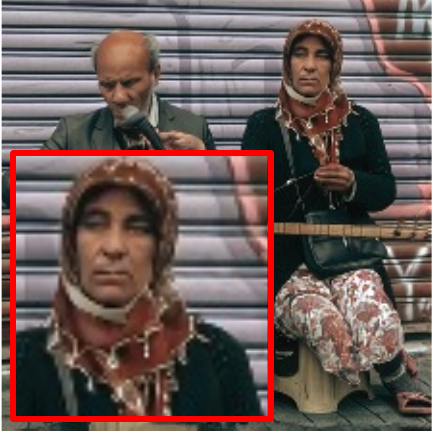} &
        \adjincludegraphics[clip,width=0.095\linewidth,trim={0 0 0 0}]{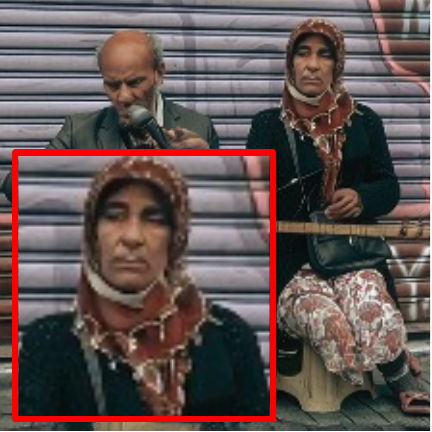} &
        \adjincludegraphics[clip,width=0.095\linewidth,trim={0 0 0 0}]{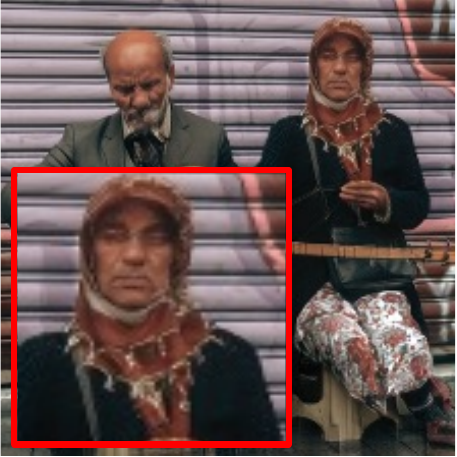} &
        \adjincludegraphics[clip,width=0.095\linewidth,trim={0 0 0 0}]{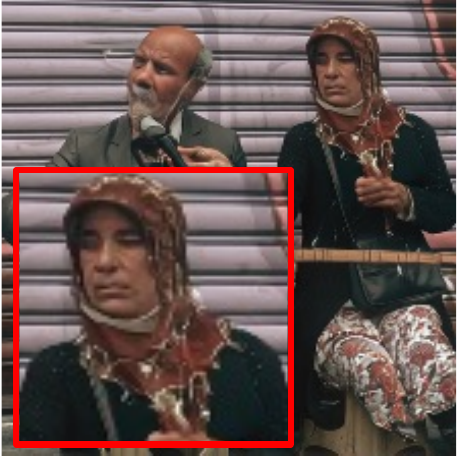} &
        \adjincludegraphics[clip,width=0.095\linewidth,trim={0 0 0 0}]{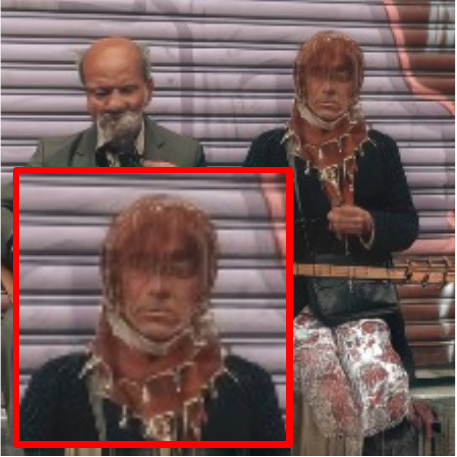} &
        \adjincludegraphics[clip,width=0.095\linewidth,trim={0 0 0 0}]{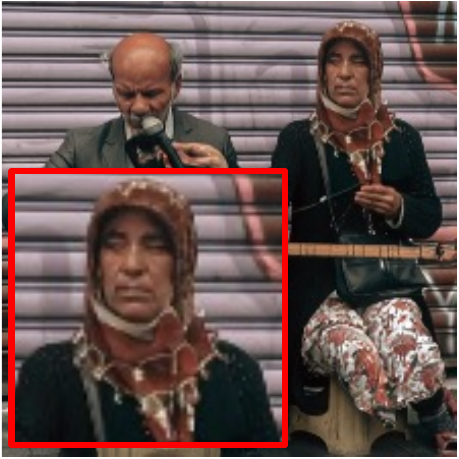} &
        \adjincludegraphics[clip,width=0.095\linewidth,trim={0 0 0 0}]{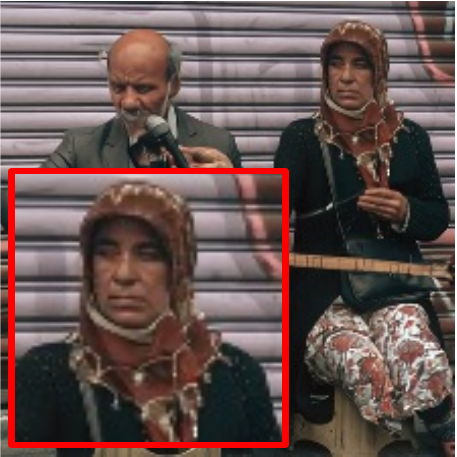} &
        \adjincludegraphics[clip,width=0.095\linewidth,trim={0 0 0 0}]{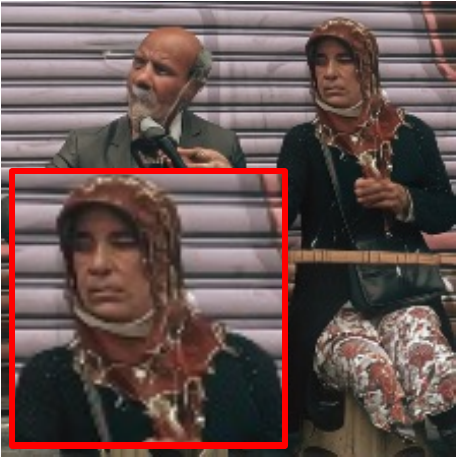} \\
        \multicolumn{1}{c}{\small Input Image} & \multicolumn{3}{c}{\small CogVideoX~\cite{yang2024cogvideox}} & \multicolumn{3}{c}{\small + Best-of-N} & \multicolumn{3}{c}{\small + SteerX (early)}
    \end{tabular}
    \vspace{-1mm}
    \vspace{\abovefigcapmargin}
    \caption{\textbf{Qualitative results of video generation in VBench-I2V~\cite{huang2024vbench}.} SteerX enhances vividness and overall visual quality.}
    \vspace{\belowfigcapmargin}
    \label{fig:cog}
\end{figure*}

%% file: figure/image-to-4d.tex
\begin{figure*}[t]
    \centering
    \includegraphics[width=\linewidth]{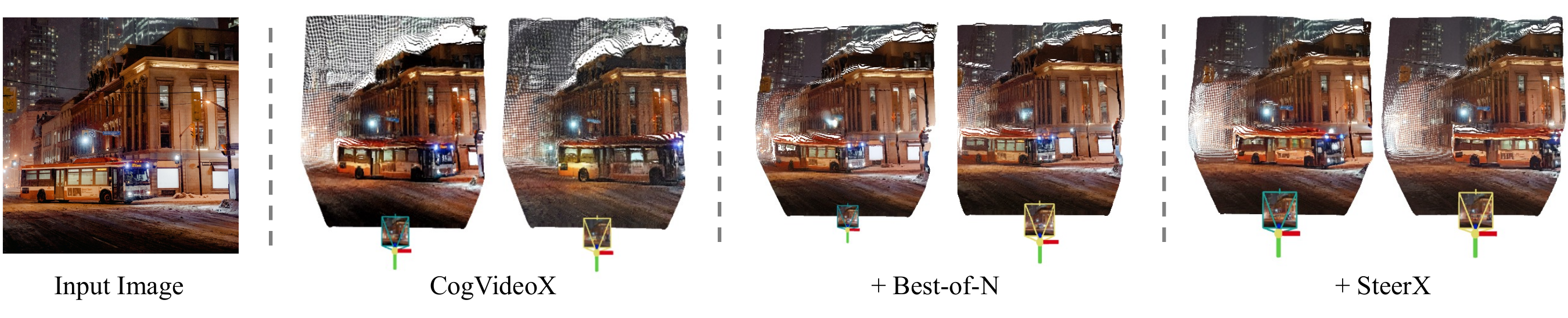}
    \vspace{-0.8cm}
    \caption{\textbf{Qualitative results in Image-to-4D.} SteerX naturally lifts object motion into 4D spaces, while preserving geometric alignments.}
    \vspace{\belowfigcapmargin}
    \label{fig:image_to_4d}
\end{figure*}

%% file: figure/qual_4d.tex
\begin{figure}[t]
    \centering
    \includegraphics[width=\linewidth]{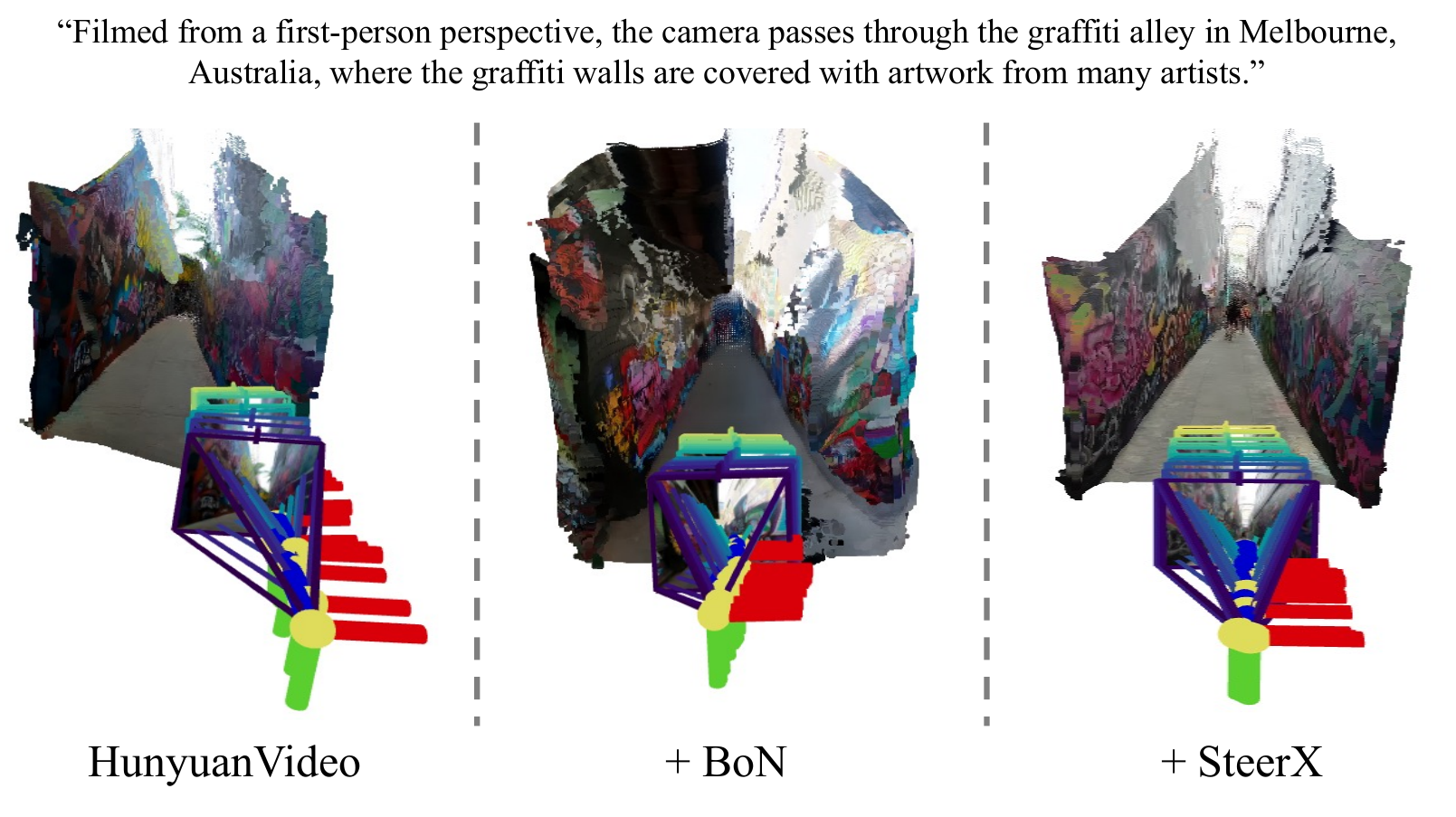}
    \vspace{-0.7cm}
    \caption{\textbf{Qualitative results in Text-to-4D.}}
    \vspace{\belowfigcapmargin}
    \label{fig:qual_4d}
\end{figure}

%% file: table/vbench_4d.tex
\begin{table}[t]
    \centering
    \setlength\tabcolsep{2pt}
    \resizebox{\linewidth}{!}{
    \begin{tabular}{lcccccc}
       \toprule
        Method & \, $k$ \,  & Aesthetic$\uparrow$ & Subject$\uparrow$ & Dynamic$\uparrow$ & Dyn-MEt3R$\uparrow$ \\
       \midrule
        
        
        CogVideoX~\cite{yang2024cogvideox} & 1 & 0.592 & \underline{0.945} & 0.158 & 0.880 \\
        \arrayrulecolor{gray}\midrule
        \; + BoN & 2 & 0.591 & 0.941 & 0.141 & 0.882 \\
        \rowcolor{gray!25} \; + \textbf{SteerX} (early) & 2 & 0.593 & \underline{0.945} & \underline{0.161} & 0.894 \\
        \midrule
        \; + BoN & 4 & \underline{0.594} & 0.944 & 0.143 & \underline{0.901} \\
        \rowcolor{gray!25} \; + \textbf{SteerX} (early) & 4 & \textbf{0.596} & \textbf{0.957} & \textbf{0.170} & \textbf{0.909}  \\
       \arrayrulecolor{black}\bottomrule
    \end{tabular}
    }
    \vspace{-1mm}
    \vspace{\abovetabcapmargin}
    \caption{\textbf{Quantitative results in VBench-I2V~\cite{huang2024vbench}.} Results demonstrate the scalability of geometric steering as $k$ increases.}
    \vspace{-2mm}
    \label{tab:vbench_4d}
\end{table}

%% file: table/vbench_3d.tex
\begin{table}[t]
    \centering
    \setlength\tabcolsep{4pt}
    \resizebox{\linewidth}{!}{
    \begin{tabular}{lcccccc}
       \toprule
        \multirow{2}{*}{Method} & \multirow{2}{*}{\, $k$ \,} & \multicolumn{3}{c}{Rendering} & Sample \\  
        \arrayrulecolor{gray}\cmidrule(lr){3-5} \cmidrule(lr){6-6}
        & & BRISQUE$\downarrow$ & NIQE$\downarrow$ & CLIP-I$\uparrow$ & GS-MEt3R$\uparrow$ \\
        \arrayrulecolor{black}\midrule
        DimensionX~\cite{sun2024dimensionx} & 1 & 37.3 & \underline{4.25} & 82.4 & 0.708 \\
        \arrayrulecolor{gray}\midrule
        \; + BoN & 4 & \underline{29.8} & 4.33 & \underline{83.2} & \underline{0.745} \\
        \rowcolor{gray!25} \; + \textbf{SteerX} (early) & 4 & \textbf{29.7} & \textbf{4.24} & \textbf{83.7} & \textbf{0.749} \\
       \arrayrulecolor{black}\bottomrule
    \end{tabular}
    }
    \vspace{-1mm}
    \vspace{\abovetabcapmargin}
    \caption{\textbf{Quantitative results in VBench-I2V~\cite{huang2024vbench}.} For rendered images, we report results after the refinement process in~\cite{li2025director3d, go2024splatflow}.}
    \vspace{\belowtabcapmargin}
    \label{tab:vbench_3d}
\end{table}

%% file: figure/splatflow.tex
\begin{figure*}[t!]
    \centering
    \setlength\tabcolsep{1.3pt}
    \begin{tabular}{cccc:cccc}
        \raisebox{0.002\linewidth}{\rotatebox{90}{\small SplatFlow~\cite{go2024splatflow}}} 
        \adjincludegraphics[clip,width=0.115\linewidth,trim={0 0 0 0}]{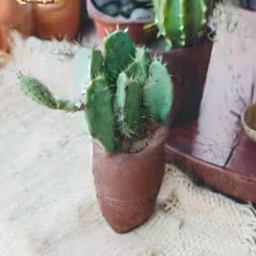} &
        \adjincludegraphics[clip,width=0.115\linewidth,trim={0 0 0 0}]{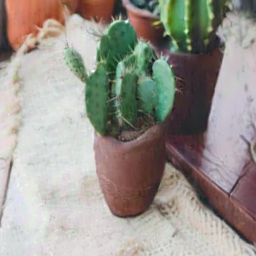} &
        \adjincludegraphics[clip,width=0.115\linewidth,trim={0 0 0 0}]{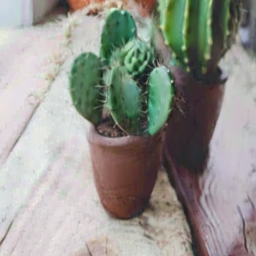} &
        \adjincludegraphics[clip,width=0.115\linewidth,trim={0 0 0 0}]{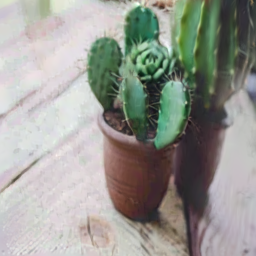} &
        \adjincludegraphics[clip,width=0.115\linewidth,trim={0 0 0 0}]{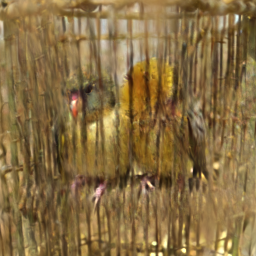} &
        \adjincludegraphics[clip,width=0.115\linewidth,trim={0 0 0 0}]{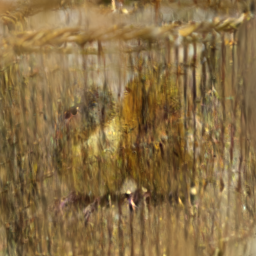} &
        \adjincludegraphics[clip,width=0.115\linewidth,trim={0 0 0 0}]{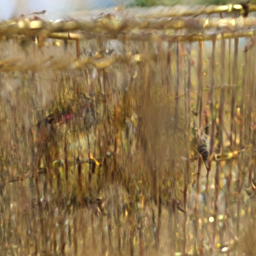} &
        \adjincludegraphics[clip,width=0.115\linewidth,trim={0 0 0 0}]{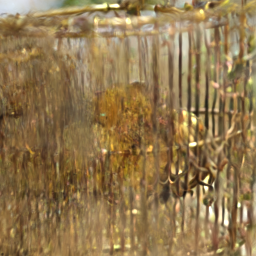}\\
        
        \raisebox{0.03\linewidth}{\rotatebox{90}{\small + BoN}} 
        \adjincludegraphics[clip,width=0.115\linewidth,trim={0 0 0 0}]{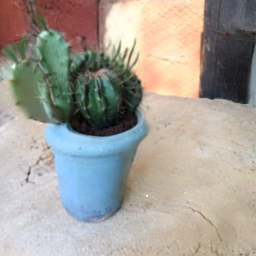} &
        \adjincludegraphics[clip,width=0.115\linewidth,trim={0 0 0 0}]{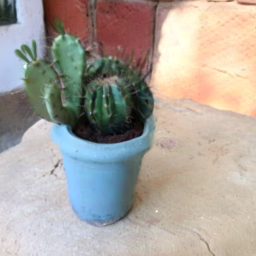} &
        \adjincludegraphics[clip,width=0.115\linewidth,trim={0 0 0 0}]{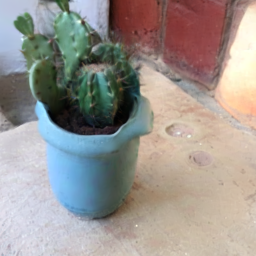} &
        \adjincludegraphics[clip,width=0.115\linewidth,trim={0 0 0 0}]{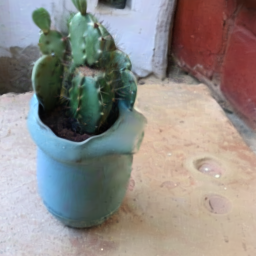} &
        \adjincludegraphics[clip,width=0.115\linewidth,trim={0 0 0 0}]{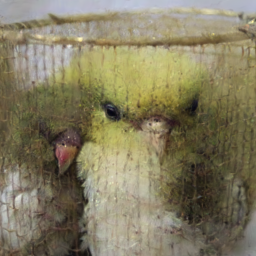} &
        \adjincludegraphics[clip,width=0.115\linewidth,trim={0 0 0 0}]{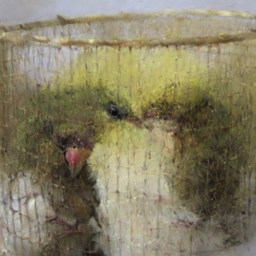} &
        \adjincludegraphics[clip,width=0.115\linewidth,trim={0 0 0 0}]{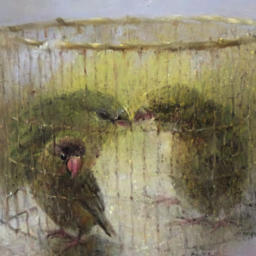} &
        \adjincludegraphics[clip,width=0.115\linewidth,trim={0 0 0 0}]{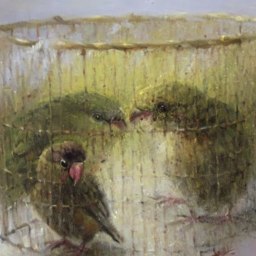}\\

        \raisebox{0.02\linewidth}{\rotatebox{90}{\small + SteerX}} 
        \adjincludegraphics[clip,width=0.115\linewidth,trim={0 0 0 0}]{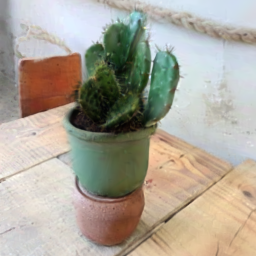} &
        \adjincludegraphics[clip,width=0.115\linewidth,trim={0 0 0 0}]{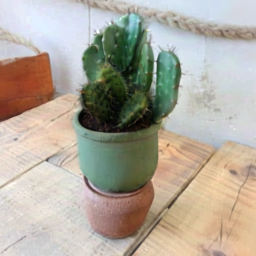} &
        \adjincludegraphics[clip,width=0.115\linewidth,trim={0 0 0 0}]{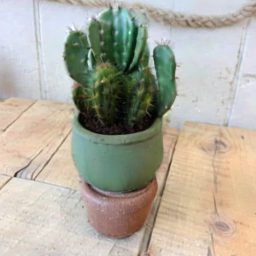} &
        \adjincludegraphics[clip,width=0.115\linewidth,trim={0 0 0 0}]{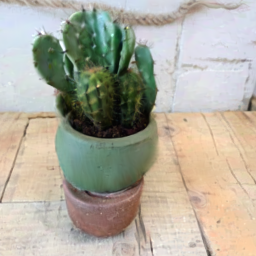} &
        \adjincludegraphics[clip,width=0.115\linewidth,trim={0 0 0 0}]{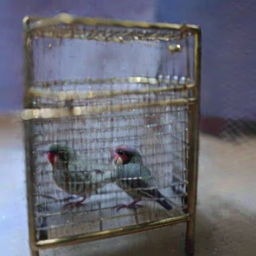} &
        \adjincludegraphics[clip,width=0.115\linewidth,trim={0 0 0 0}]{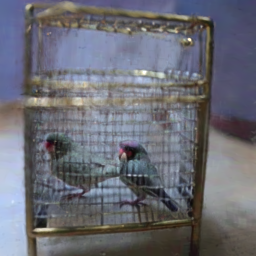} &
        \adjincludegraphics[clip,width=0.115\linewidth,trim={0 0 0 0}]{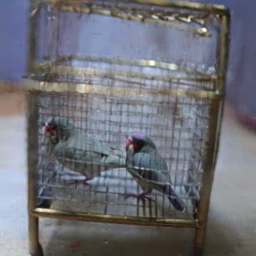} &
        \adjincludegraphics[clip,width=0.115\linewidth,trim={0 0 0 0}]{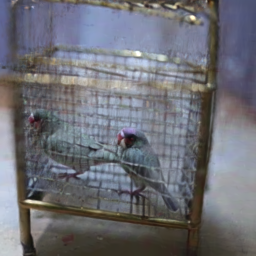}\\
        
        \multicolumn{4}{p{0.47\linewidth}}{\centering \textit{A green cactus in a clay pot.}} & \multicolumn{4}{p{0.47\linewidth}}{\centering \textit{A pair of lovebirds in a golden cage.}} \\
    \end{tabular}
    \vspace{-1mm}
    \vspace{\abovefigcapmargin}
    \caption{\textbf{Qualitative results of rendered images in T3Bench~\cite{he2023t3bench}.} SteerX significantly enhances the visual quality of rendered images and textual alignment, demonstrating its compatibility with various video generation and scene reconstruction models.}
    \vspace{\belowfigcapmargin}
    \label{fig:splatflow}
\end{figure*}

%% file: figure/dimensionx.tex
\begin{figure}[t]
    \centering
    \includegraphics[width=0.95\linewidth]{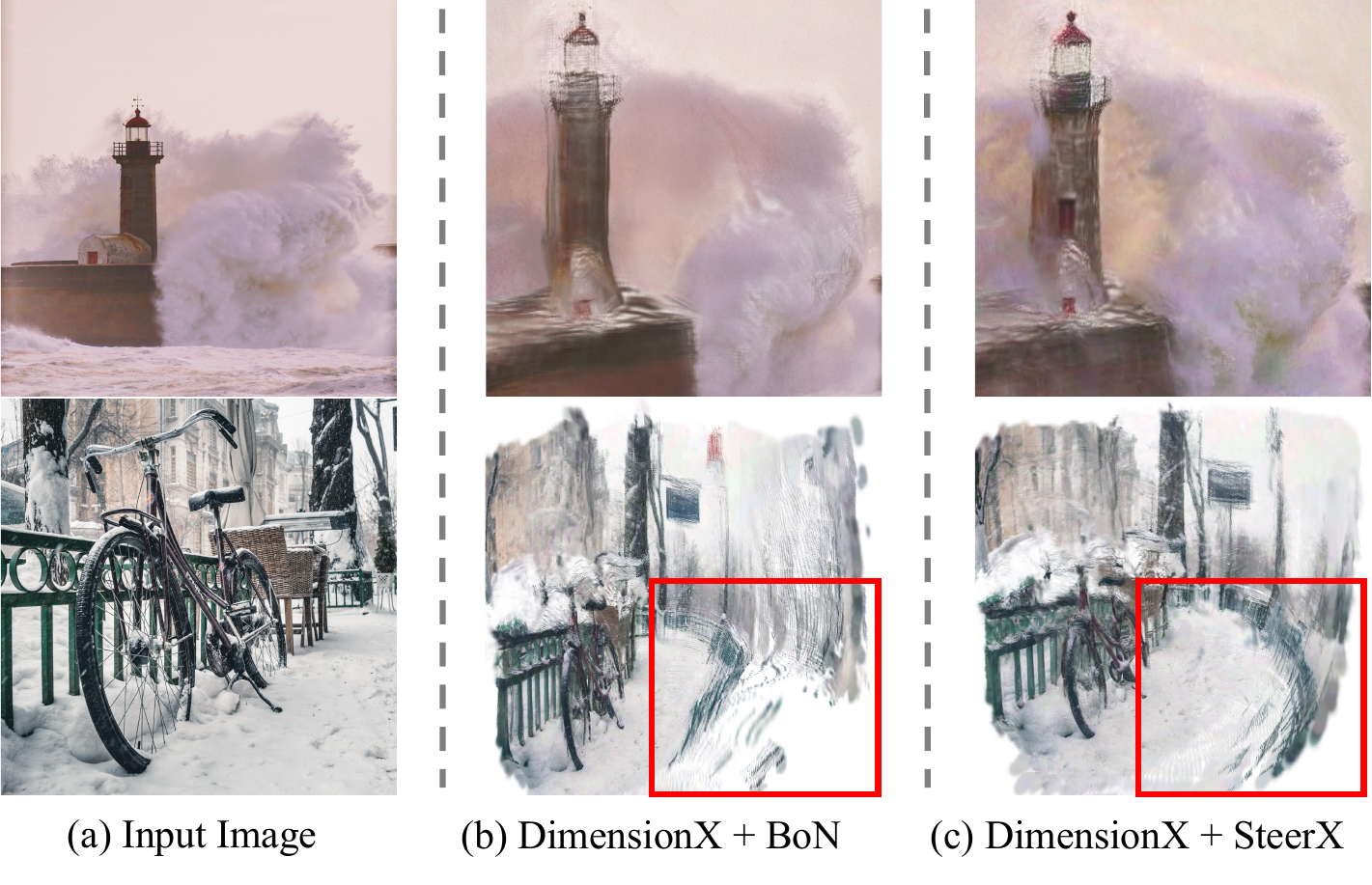}
    \vspace{-0.35cm}
    \caption{\textbf{Qualitative results in of 3DGS in VBench-I2V~\cite{huang2024vbench}.}}
    \vspace{-3mm}
    \label{fig:dimensionx}
\end{figure}

%% file: table/t3bench.tex
\begin{table}[t]
    \centering
    \setlength\tabcolsep{4pt}
    \resizebox{\linewidth}{!}{
    \begin{tabular}{lcccc}
       \toprule
        Method & \, $k$ \, & BRISQUE$\downarrow$ & CLIPScore$\uparrow$ & GS-MEt3R$\uparrow$ \\
        \midrule
        SplatFlow~\cite{go2024splatflow} & 1 & 33.25  & \underline{29.66} & 0.727 \\
        \arrayrulecolor{gray}\midrule
        \; + BoN & 2 & 36.82 & 29.55 & 0.756 \\
        \rowcolor{gray!25} \; + \textbf{SteerX} (early) & 2 & \underline{27.47} & 29.56 & 0.767 \\
        \midrule
        \; + BoN & 4 & 28.14 & 29.28 & \underline{0.768} \\
        \rowcolor{gray!25} \; + \textbf{SteerX} (early) & 4 & \textbf{26.65} & \textbf{29.73} & \textbf{0.775} \\
       \arrayrulecolor{black}\bottomrule
    \end{tabular}
    }
    \vspace{-1mm}
    \vspace{\abovetabcapmargin}
    \caption{\textbf{Quantitative results of multi-view generation in Single-Object-with-Surrounding set of T3Bench~\cite{he2023t3bench}.}}
    \vspace{\belowtabcapmargin}
    \vspace{-1mm}
    \label{tab:t3bench}
\end{table}

%% file: table/t3bench_render.tex
\begin{table}[t]
    \centering
    \setlength\tabcolsep{7pt}
    \resizebox{\linewidth}{!}{
    \begin{tabular}{lccc}
       \toprule
        Method   & BRISQUE$\downarrow$ & NIQE$\downarrow$ & CLIPScore$\uparrow$ \\
       \midrule
       DreamFusion~\cite{poole2022dreamfusion} & 90.2 & 10.48 & -   \\
       Magic3D~\cite{lin2023magic3d} & 92.8 & 11.20 & -   \\
       LatentNeRF~\cite{metzer2023latent} & 88.6 & 9.19 & - \\
       SJC~\cite{wang2023score} & 82.0 & 10.15 & - \\
       Fantasia3D~\cite{chen2023fantasia3d} & 69.6 & 7.65 & - \\
       ProlificDreamer~\cite{wang2024prolificdreamer} & 61.5 & 7.07 & - \\
       Director3D~\cite{li2025director3d} & 32.3 & 4.35 & 32.9 \\
       SplatFlow~\cite{go2024splatflow} & 19.6 & \textbf{4.24} & \underline{33.2} \\
       \arrayrulecolor{gray}\midrule
       SplatFlow~\cite{go2024splatflow}$^\dagger$ & 23.4 & 4.84 & 32.7 \\
       \midrule
       \; + BoN ($k=4$)& \underline{17.2} & 4.41 & 32.3 \\
       \rowcolor{gray!25} \; + \textbf{SteerX} ($k=4$) & \textbf{13.1} & \underline{4.30} & \textbf{33.4} \\
       \arrayrulecolor{black}\bottomrule
    \end{tabular}
    }
    \vspace{-1mm}
    \vspace{\abovetabcapmargin}
    \caption{\textbf{Quantitative results in T3Bench~\cite{he2023t3bench}.} SplatFlow$^\dagger$ is re-implemented results without the stop ray strategy, which is incompatible with our geometric steering pipeline.}
    \vspace{\belowtabcapmargin}
    \label{tab:t3bench_render}
\end{table}

%% file: sec/6_conclusion.tex
\section{Conclusion}
\label{sec:conclusion}

In this paper, we have introduced SteerX, a zero-shot inference-time steering method for camera-free 3D and 4D scene generation. Instead of addressing physical alignment separately in either video generation or scene reconstruction, SteerX unifies both stages and iteratively tilts the data distribution toward geometrically consistent samples. Extensive experiments across diverse scene generation tasks verify that SteerX effectively enhances visual quality, textual alignment, and geometric consistency. SteerX is practical, enabling zero-shot 3D and 4D scene generation, and is scalable, with geometric alignment improving as the number of particles increases. We believe this efficient scaling property holds great potential, opening new avenues for 3D and 4D scene generation, particularly in test-time scaling.

%% file: sec/_appendix.tex
\appendix

\section{Proofs}
\label{sec:proof}

\convergence*
\begin{proof}
    From the conditions, the unnormalized weight assigned to a complete path is
    \begin{align}
        W_{\x_{T:0}} &= \prod_{t=1}^T \left[
        (1 + \epsilon_t(\x_{T:t}))G_t(\x_{T:t})
        \right]G_0(\x_{T:0}) \\
        &= \exp (\lambda r_\phi(\hat\x_0))\prod_{t=1}^T (1 + \epsilon_t(\x_{T:t})),
        \label{eq:definition_W}
    \end{align}
    where for the second equality, we used
    \begin{align}
        G_0(\x_{T:0}) \prod_{t=1}^T G_t(\x_{T:t}) = \exp (\lambda r_\phi(\hat\x_0)).
    \end{align}
    Let $r_\phi(\hat\x_0) = r_\phi(\x_0) + \mathbf{\delta}(\x_0)$. We have
    \begin{align}
        \exp (\lambda r_\phi(\hat\x_0)) = \exp (\lambda r_\phi(\x_0)) \exp (\lambda \mathbf{\delta}(\x_0)).
    \end{align}
    Given $|\mathbf{\delta}(\x_0)| \leq \eta$, we use the Taylor expansion
    \begin{align}
        \exp (\lambda \mathbf{\delta}(\x_0)) = 1 + \mathcal{O}(\lambda \eta).
    \label{eq:error_eta}
    \end{align}
    Further, we have that
    \begin{align}
        \prod_{t=1}^T (1 + \epsilon_t(\x_{T:t})) = 1 + \mathcal{O}(T\varepsilon).
    \label{eq:error_epsilon}
    \end{align}
    Combining \eqref{eq:definition_W},\eqref{eq:error_eta}, and \eqref{eq:error_epsilon}, the full weight reads
    \begin{align}
        W(\x_{T:0}) = \exp (\lambda r_\phi(\x_0)) (1 + \mathcal{O}(T\varepsilon + \lambda\eta)).
    \end{align}
    Integrating out the latent variables $\x_{T:1}$, the proof is complete.
\end{proof}

\input{algo/rf_steer}

\section{Geometric steering on rectified flow models}

Rectified flow-based video generative models~\cite{genmo2024mochi, kong2024hunyuanvideo, go2024splatflow} follow a straight Ordinary Differential Equation path, making it challenging to apply geometric steering since resampling particles does not introduce diverse sampling trajectories. Therefore, to introduce a stochastic process into the generation process, we provide additional modifications to adapt geometric steering for rectified flow models, as shown in~\cref{alg:geo_rf_steering}. The process of computing intermediate rewards and potentials remains the same as before. However, instead of resampling new particles from the existing particles, we resample the expected $\hat{x}_{t_0}$ from the multinomial distribution. Then, project the resampled particles onto a valid manifold at each noise level. This approach effectively enables geometric steering in rectified flow models and ensures that the model explores diverse trajectories.

\section{Additional Results}

We present additional experiments and results to further validate the scalability and effectiveness of SteerX. In Section~\ref{subsec:scalable}, we explore how increasing the number of particles or extending video length impacts geometric steering, providing insights into the scaling properties of SteerX. We show qualitative comparisons for Text-to-4D generation in Section~\ref{subsec:t4d}, and additional qualitative results for both Text-to-4D and Image-to-3D scene generation in Section~\ref{subsec:qual_results}.

\input{rebuttal/resampling}

\input{table/mochi_abl}

\input{table/long_video_hunyuan}

\subsection{Analysis on design choices} Linear resampling places resampling steps at uniform intervals across the entire timestep $T$. Also, as shown in \cref{fig:resampling}, the generative model tends to form a coarse geometric structure around $0.8T$ and focuses on fine details after $0.6T$. Based on this observation, early and late resampling are uniformly scheduled between $0.8T$~–~$0.6T$ and $0.4T$~–~$0.2T$, respectively. Early resampling allows the model to build upon the coarse structure, refine local geometry, and gradually incorporate fine details by exploring diverse generation trajectories. In contrast, reward values tend to plateau in the later steps, indicating limited exploration at late resampling.

\input{rebuttal/execution_time}

\subsection{Scalability of SteerX}
\label{subsec:scalable}

We further explore the scaling property of SteerX by increasing the number of particles $k$ and video length $N$. \Cref{fig:runtime} presents the execution time versus reward values for all generation tasks as the number of particles increases. Although SteerX incurs additional computational overhead by forwarding the scene reconstruction model multiple times, it demonstrates better inference-time scalability than BoN. Also, as the number of particles increases, SteerX achieves greater performance gains by exploring more diverse sampling trajectories, rather than relying on post-hoc selection. \Cref{tab:scalable_particle} presents quantitative results on the performance of 4D scene generation as the number of particles increases. We observe that Dyn-MEt3R remains highly correlated with other evaluation metrics, further demonstrating the robustness of SteerX's scalability.
Also, \cref{appfig:long} and \Cref{tab:scalable_frame} show the impact of extending video length on Text-to-4D scene generation. We observe that as video length increases, the generated videos become more dynamic and tend to be more object-centric. Compared to the best-of-N approach, SteerX generates more visually plausible and dynamic objects, effectively capturing camera motion.

\subsection{Additional comparisons in Text-to-4D}
\label{subsec:t4d}

We further present qualitative comparisons to demonstrate the effectiveness of SteerX in following the given camera descriptions, as shown in \Cref{fig:qual_comp_4d}. SteerX successfully aligns with both the specified camera trajectories and object motions, resulting in highly natural 4D scenes.

\input{rebuttal/video}

\subsection{Additional qualitative results}
\label{subsec:qual_results}

As shown in~\Cref{fig:qual_3d,fig:supp_text_to_4d,fig:supp_text_to_4d_2}, we provide additional qualitative results for Text-to-4D and Image-to-3D scene generation, demonstrating SteerX’s ability to generate diverse 3D and 4D scenes only from images or text prompts. We also provide video results in~\cref{fig:demo}.

\section{Limitations and Discussions}
While SteerX effectively enhances both visual quality and geometric alignment in 3D and 4D scene generation, it has certain limitations that present opportunities for future improvements. First, SteerX currently relies on existing feed-forward scene reconstruction models, meaning it cannot directly reconstruct 4D Gaussian Splats (4DGS). Second, video generative models for 4D scene generation struggle to produce video frames with large inter-frame camera motion, limiting the overall scene scale. Future advancements in video generation models that better handle broad camera motion ranges will further enhance SteerX's effectiveness in large-scale 4D scene generation.

\input{supp_materials/qual_text_to_4d_long}

\input{supp_materials/qual_3d}

\input{supp_materials/qual_text_to_4d_pointcloud}

\input{supp_materials/qual_4d_comp}

%% file: algo/rf_steer.tex
\begin{figure}[t]
\vspace{-1em}
\begin{algorithm}[H]
    \textbf{Required:} rectified flow model $\mathbf{v}_\theta$, reward function $r_\phi$, number of particles $k$, and initial noise $\{\mathbf{x}^{j}_{t_N}\}^{k}_{j=1} \sim \mathcal{N}(0, I)$.
    \caption{SteerX (rectified flow)}\label{alg:geo_rf_steering} 
    \textbf{Sampling:}
    \begin{algorithmic}[1]
        \For{$i \in \{N - 1, \dotsc, 0\}$}
            \vspace{1mm}
            \For{$j \in \{1 \dotsc k\}$}
                \vspace{1mm}
                \State $\hat{\mathbf{x}}^{j}_{t_0} \gets \mathbf{x}^{j}_{t_{i+1}} - t_{i+1} \mathbf{v}_\theta(\mathbf{x}^{j}_{t_{i+1}})$
                \vspace{1mm}
                \State $\mathbf{s}^{j}_{t_i} \gets r_\phi(\hat{\mathbf{x}}^{j}_{t_0}) \hfill \triangleright \textit{Intermediate rewards}$
                \vspace{1mm}
                \State ${G}^{j}_{t_i} \gets \exp(\lambda\max^{t_N}_{l=t_i}(\mathbf{s}^{j}_{l})) \hfill \triangleright \textit{Potential}$
                \vspace{1mm}
            \EndFor
            \vspace{1mm}
            \State $\{\hat{\mathbf{x}}^j_{t_0}\}^k_{j=1} \sim \text{Multinomial}(\{\hat{\mathbf{x}}^j_{t_0}, G^j_{t_i}\}^k_{j=1}$)  
            \vspace{1mm}
            \State $\bm{z} \sim \mathcal{N}(0, I)$
            \vspace{1mm}
            \State $\mathbf{x}^{j}_{t_i} \gets (1 - t_i)\{\hat{\mathbf{x}}^j_{t_0}\}^k_{j=1} +t_i \bm{z} $
        \vspace{1mm}
        \EndFor
        \vspace{1mm}
        \State $l \gets \argmax_{i \in \{1, \dotsc, k\}} r_\phi(\mathbf{x}^i_{t_0})$
        \State \textbf{return} $\mathbf{x}^l_{t_0}$
    \end{algorithmic}
\end{algorithm}
\vspace{-2.2em}
\end{figure}

%% file: rebuttal/resampling.tex
\begin{figure}[t]
    \centering
    \includegraphics[width=\linewidth]{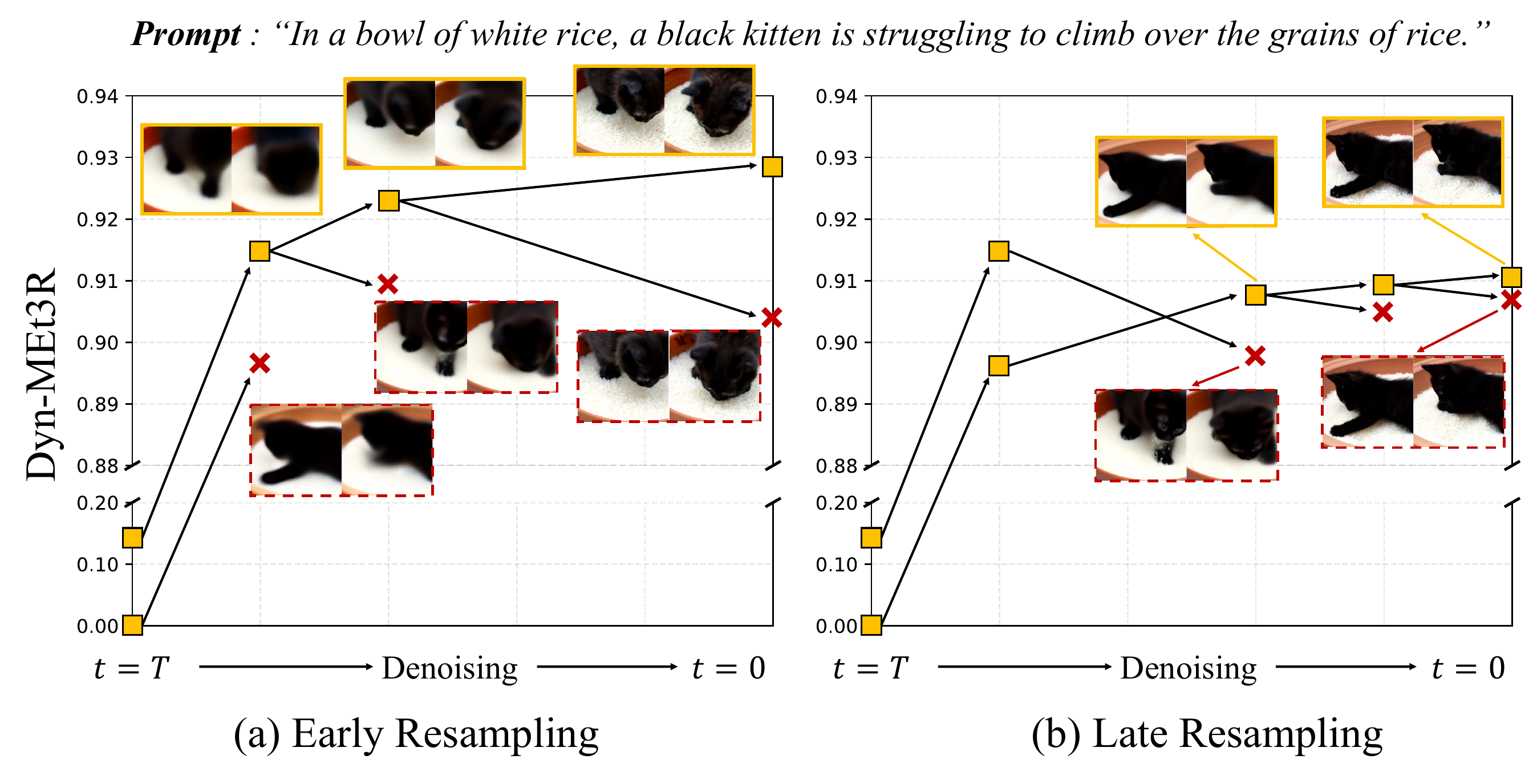}
    \vspace{-0.7cm}
    \caption{\textbf{Resampling analysis} for $k=2, M=2$ in Text-to-4D.}
    \vspace{-0.1cm}
    \label{fig:resampling}
\end{figure}

%% file: table/mochi_abl.tex
\begin{table}[t]
    \centering
    \setlength\tabcolsep{6pt}
    \resizebox{\linewidth}{!}{
    \begin{tabular}{lccccc}
       \toprule
        Method & \, $k$ \, & Aesthetic$\uparrow$  & Temporal$\uparrow$ & Dynamic$\uparrow$ & Dyn-MEt3R$\uparrow$ \\
       \midrule
        Mochi~\cite{genmo2024mochi} & 1 & 0.491 & 0.243 & - & 0.884 \\
        \rowcolor{gray!25} \; + \textbf{SteerX}  & 4 & \underline{0.500} & \underline{0.248} & - & \underline{0.929}\\
        \rowcolor{gray!25} \; + \textbf{SteerX}  & 8  & \textbf{0.526} & \textbf{0.251} & -& \textbf{0.945}\\
        \arrayrulecolor{gray}\midrule
        HunyuanVideo~\cite{kong2024hunyuanvideo} & 1  & 0.549 & 0.241 & - & 0.911 \\
        \rowcolor{gray!25} \; + \textbf{SteerX}  & 4  & \underline{0.555} & \underline{0.243} & -& \underline{0.964}\\
        \rowcolor{gray!25} \; + \textbf{SteerX}  & 8  & \textbf{0.570} & \textbf{0.244} & - & \textbf{0.979}\\
        \midrule
        CogVideoX~\cite{yang2024cogvideox} & 1  & 0.592 & - & 0.158 & 0.880 \\
        \rowcolor{gray!25} \; + \textbf{SteerX}  & 4 & \underline{0.596} & - & \underline{0.170} &\underline{0.909}\\
        \rowcolor{gray!25} \; + \textbf{SteerX}  & 8 & \textbf{0.600} & - & \textbf{0.172} & \textbf{0.930}\\
       \arrayrulecolor{black}\bottomrule
    \end{tabular}
    }
    \vspace{-1mm}
    \vspace{\abovetabcapmargin}
    \caption{\textbf{Ablation study on the number of particles.}}
    \label{tab:scalable_particle}
\end{table}

%% file: table/long_video_hunyuan.tex
\begin{table}[t]
    \centering
    \setlength\tabcolsep{8pt}
    \resizebox{\linewidth}{!}{
    \begin{tabular}{lcccc}
       \toprule
        Method & \, $k$ \, & $N$ & Temporal$\uparrow$ & Dyn-MEt3R$\uparrow$ \\
       \midrule
        HunyuanVideo~\cite{kong2024hunyuanvideo} & 1 & 25  & 0.241 & 0.911  \\
        HunyuanVideo~\cite{kong2024hunyuanvideo} & 1 & 49  & 0.245& 0.940 \\
        \arrayrulecolor{gray}\midrule
        \; + BoN & 4 & 25  & 0.239 & 0.931\\
        \; + BoN & 4 & 49  & \underline{0.246} & 0.948 \\
        \midrule
        \rowcolor{gray!25} \; + \textbf{SteerX}  & 4 & 25  & 0.243 & \underline{0.964} \\
        \rowcolor{gray!25} \; + \textbf{SteerX}  & 4 & 49  & \textbf{0.248} & \textbf{0.978} \\
       \arrayrulecolor{black}\bottomrule
    \end{tabular}
    }
    \vspace{-1mm}
    \vspace{\abovetabcapmargin}
    \caption{\textbf{Ablation study on the number of frames.}}
    \vspace{\belowtabcapmargin}
    \label{tab:scalable_frame}
\end{table}

%% file: rebuttal/execution_time.tex
\begin{figure}[t]
    \centering
    \includegraphics[width=\linewidth]{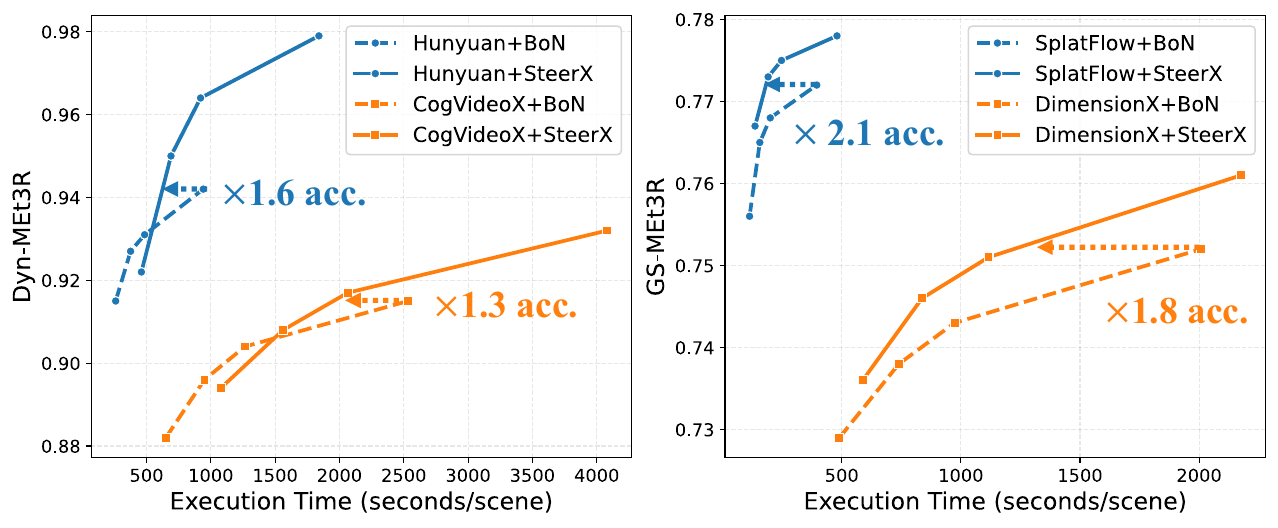}
    \vspace{-0.8cm}
    \caption{\textbf{Scalability analysis} with $k=2, 3, 4, 8$. We use 100 randomly selected samples in VBench-I2V for Image-to-3D/4D.}
    \vspace{\belowfigcapmargin}
    \vspace{-2.5mm}
    \label{fig:runtime}
\end{figure}

%% file: rebuttal/video.tex
\begin{figure*}[t]
    \centering
    \setlength\tabcolsep{4pt}
    \begin{tabular}{cccc}
        \animategraphics[poster=00, width=0.23\linewidth, height=0.23\linewidth]{8}{rebuttal/image_to_4d_1/ezgif-frame-0}{00}{16} &
        \animategraphics[poster=00, width=0.23\linewidth, height=0.23\linewidth]{8}{rebuttal/image_to_4d_2/ezgif-frame-0}{00}{16} &
        \animategraphics[poster=00, width=0.23\linewidth, height=0.23\linewidth]{8}{rebuttal/text_to_4d_1/ezgif-frame-0}{00}{15} &
        \animategraphics[poster=00, width=0.23\linewidth, height=0.23\linewidth]{8}{rebuttal/text_to_4d_2/ezgif-frame-0}{00}{15}
    \end{tabular}
    \vspace{-0.25cm}
    \caption{\textbf{4D demo.} Please click each example in Acrobat Reader.}
    \vspace{-0.35cm}
    \label{fig:demo}
\end{figure*}

%% file: supp_materials/qual_text_to_4d_long.tex
\begin{figure*}[t!]
    \centering
    \setlength\tabcolsep{2pt}
    
    \begin{tabular}{c:ccc:ccc}
      
        \multirow{2}[2]{*}[0.04\linewidth]{%
            \parbox{0.16\linewidth}{%
                \centering
                \footnotesize
                \textit{The Golden Gate Bridge glows with a warm halo in the sunset's afterglow, standing majestically against the sea breeze with ships slowly passing beneath.}%
            }%
        }  &
        \adjincludegraphics[clip,width=0.13\linewidth,trim={0 0 0 0}]{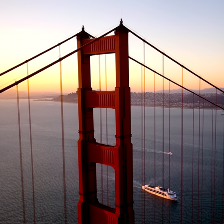} &
        \adjincludegraphics[clip,width=0.13\linewidth,trim={0 0 0 0}]{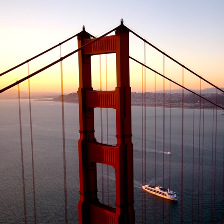} &
        \adjincludegraphics[clip,width=0.13\linewidth,trim={0 0 0 0}]{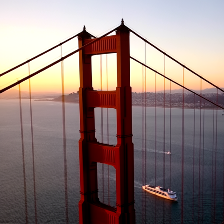} &
        \adjincludegraphics[clip,width=0.13\linewidth,trim={0 0 0 0}]{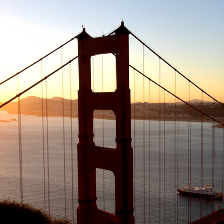} &
        \adjincludegraphics[clip,width=0.13\linewidth,trim={0 0 0 0}]{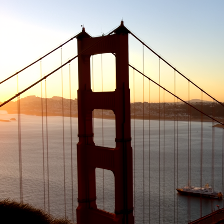} &
        \adjincludegraphics[clip,width=0.13\linewidth,trim={0 0 0 0}]{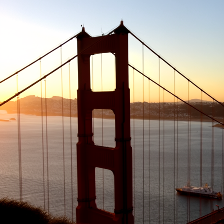} \\

        \arrayrulecolor{gray}\cmidrule(lr){2-7}
        
         &
        \adjincludegraphics[clip,width=0.13\linewidth,trim={0 0 0 0}]{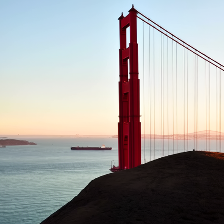} &
        \adjincludegraphics[clip,width=0.13\linewidth,trim={0 0 0 0}]{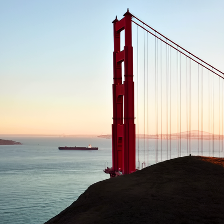} &
        \adjincludegraphics[clip,width=0.13\linewidth,trim={0 0 0 0}]{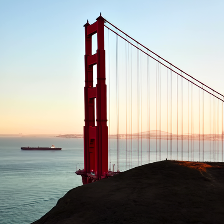} &
        \adjincludegraphics[clip,width=0.13\linewidth,trim={0 0 0 0}]{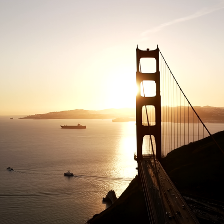} &
        \adjincludegraphics[clip,width=0.13\linewidth,trim={0 0 0 0}]{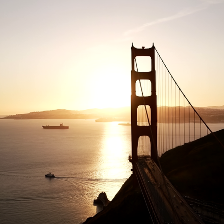} &
        \adjincludegraphics[clip,width=0.13\linewidth,trim={0 0 0 0}]{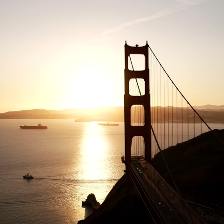} \\

        \multicolumn{1}{c}{\small Text} & \multicolumn{3}{c}{\small Hunyuan + BoN} & \multicolumn{3}{c}{\small Hunyuan + SteerX}
    \end{tabular}
    \vspace{\abovefigcapmargin}
    \caption{\textbf{Qualtitave ablation on video length.} We use four particles and visualize frames with $N=25$ (top) and $N=49$ (bottom).}
    \vspace{\belowfigcapmargin}
    \label{appfig:long}
\end{figure*}

%% file: supp_materials/qual_3d.tex
\begin{figure*}[t]
    \centering
    \includegraphics[width=\linewidth]{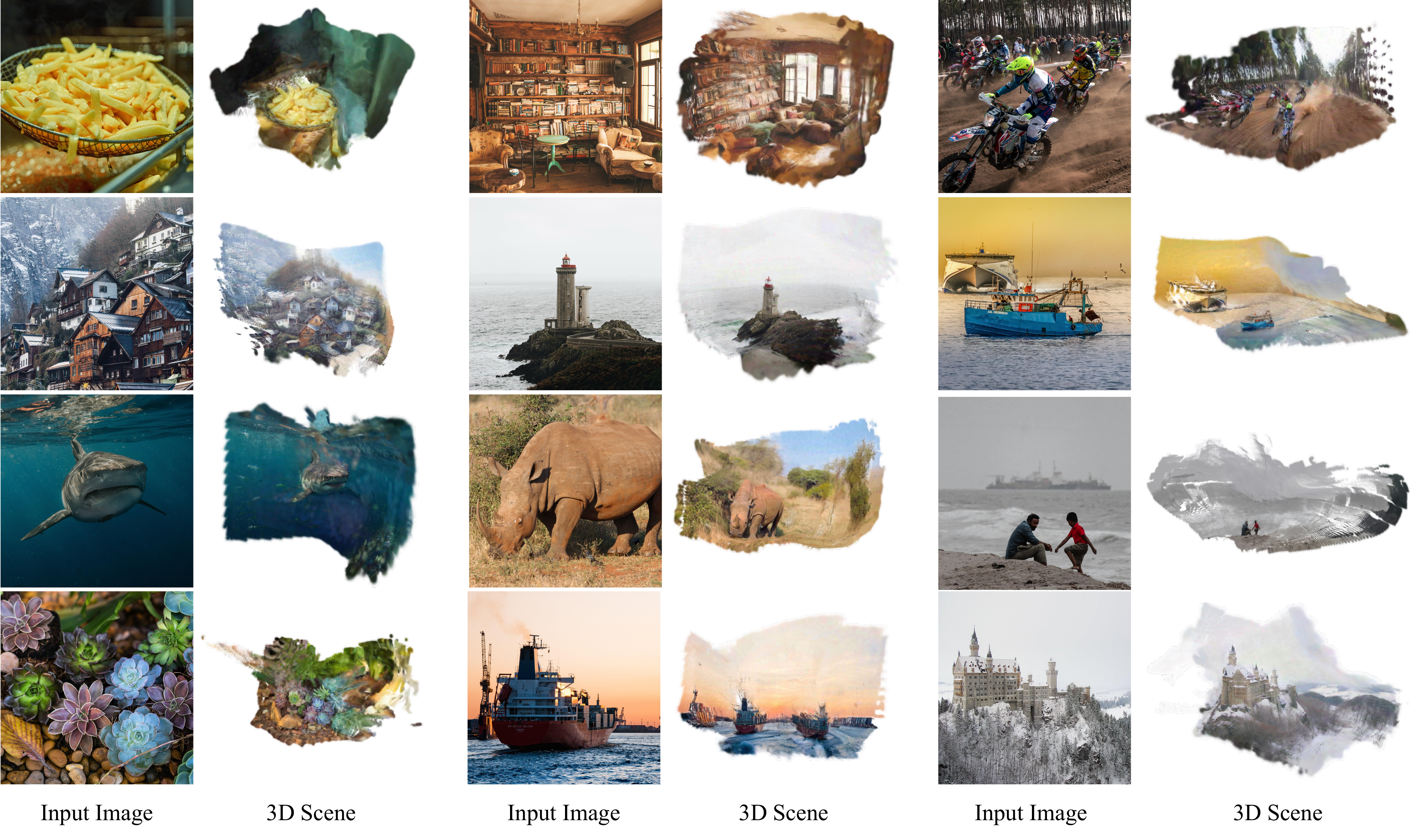}
    \caption{\textbf{Additional qualitative results in Image-to-3D.}}
    \vspace{\belowfigcapmargin}
    \label{fig:qual_3d}
\end{figure*}

%% file: supp_materials/qual_text_to_4d_pointcloud.tex
\begin{figure*}[t]
    \centering
    \includegraphics[width=\linewidth, height=0.9\textheight, keepaspectratio]{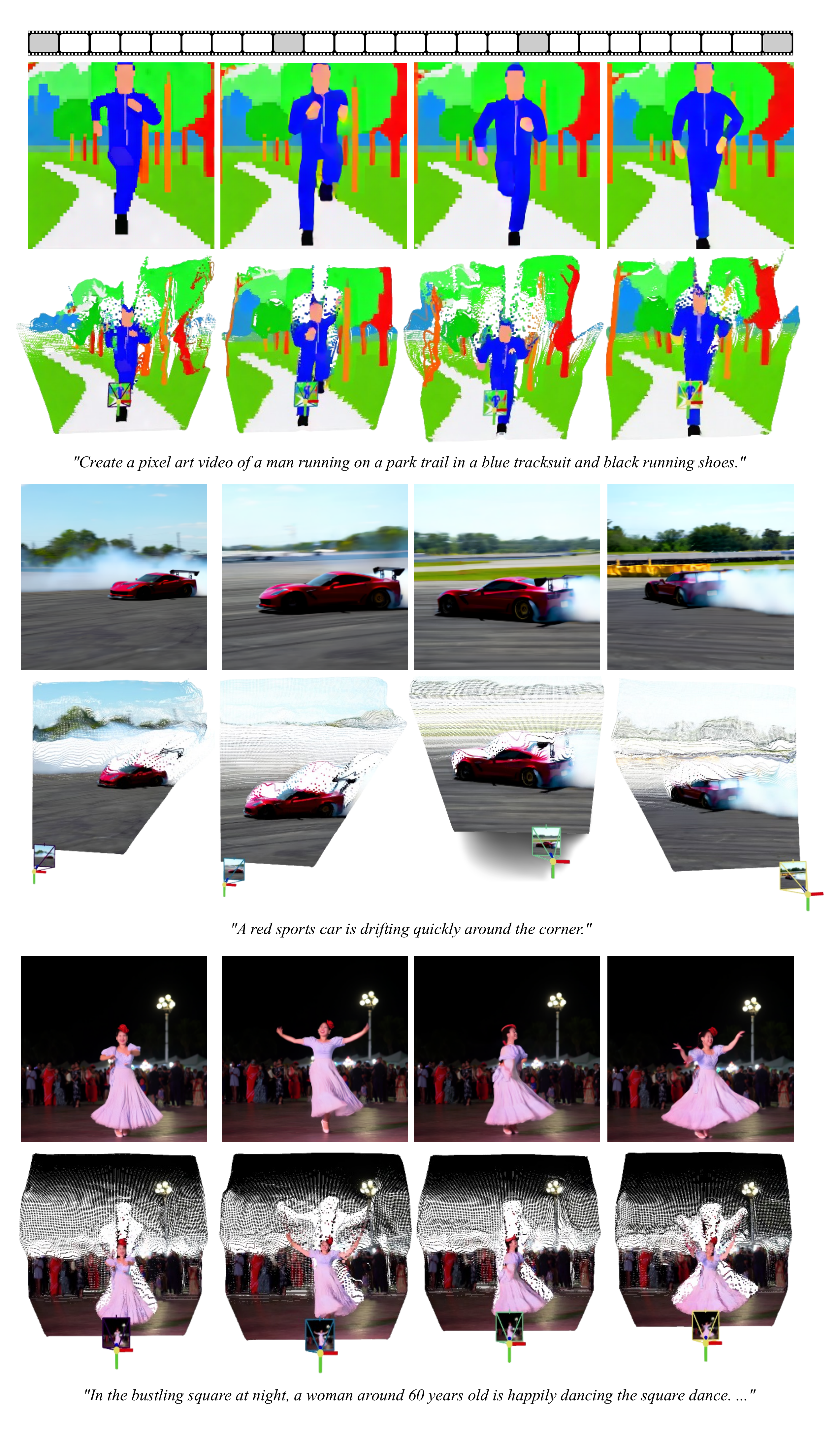}
    \caption{\textbf{Additional qualitative results in Text-to-4D.}}
    \vspace{\belowfigcapmargin}
    \label{fig:supp_text_to_4d}
\end{figure*}

\begin{figure*}[t]
    \centering
    \includegraphics[width=\linewidth, height=0.9\textheight, keepaspectratio]{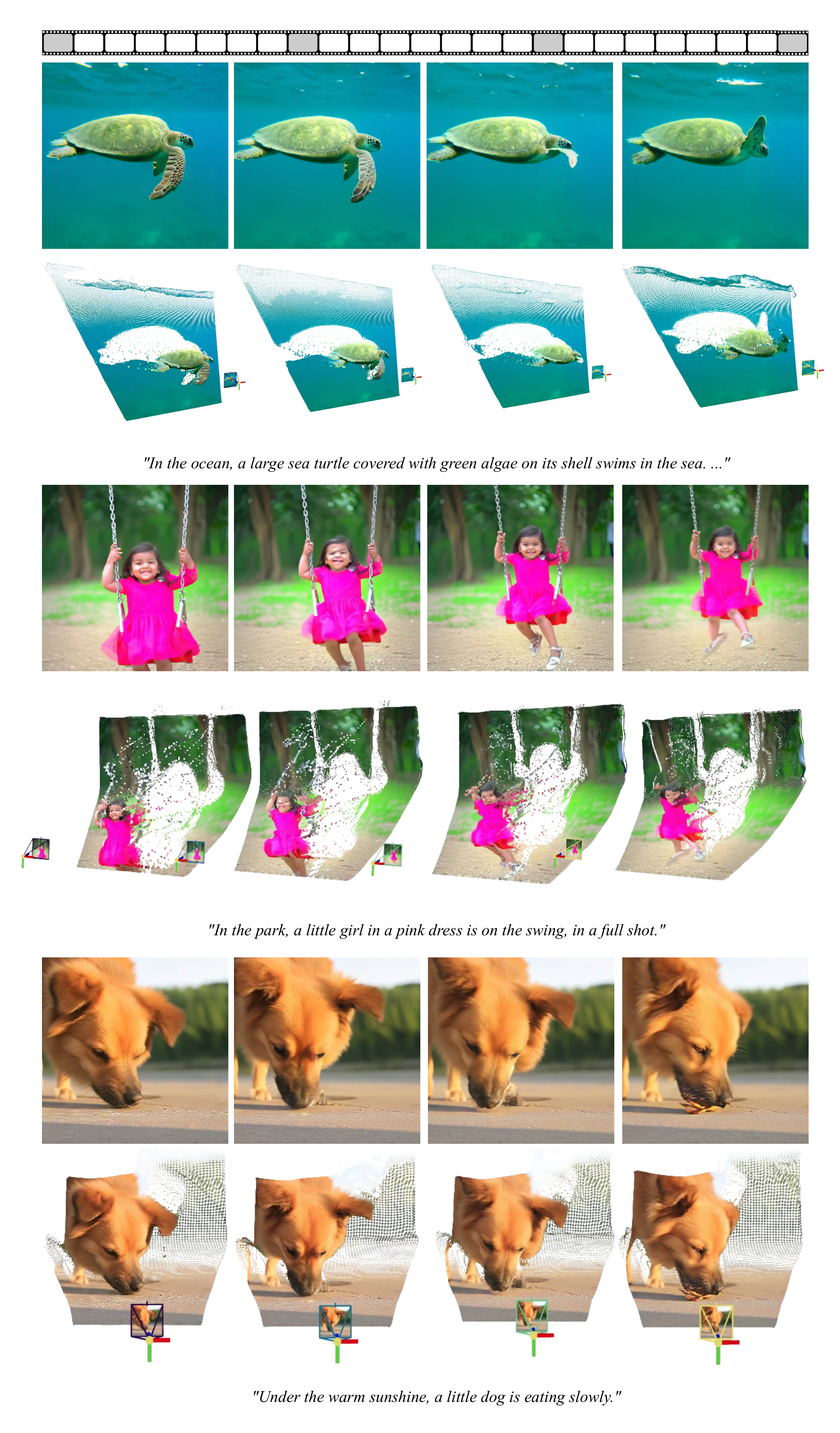}
    \caption{\textbf{Additional qualitative results in Text-to-4D.}}
    \vspace{\belowfigcapmargin}
    \label{fig:supp_text_to_4d_2}
\end{figure*}

%% file: supp_materials/qual_4d_comp.tex
\begin{figure*}[t]
    \centering
    \includegraphics[width=0.7\linewidth]{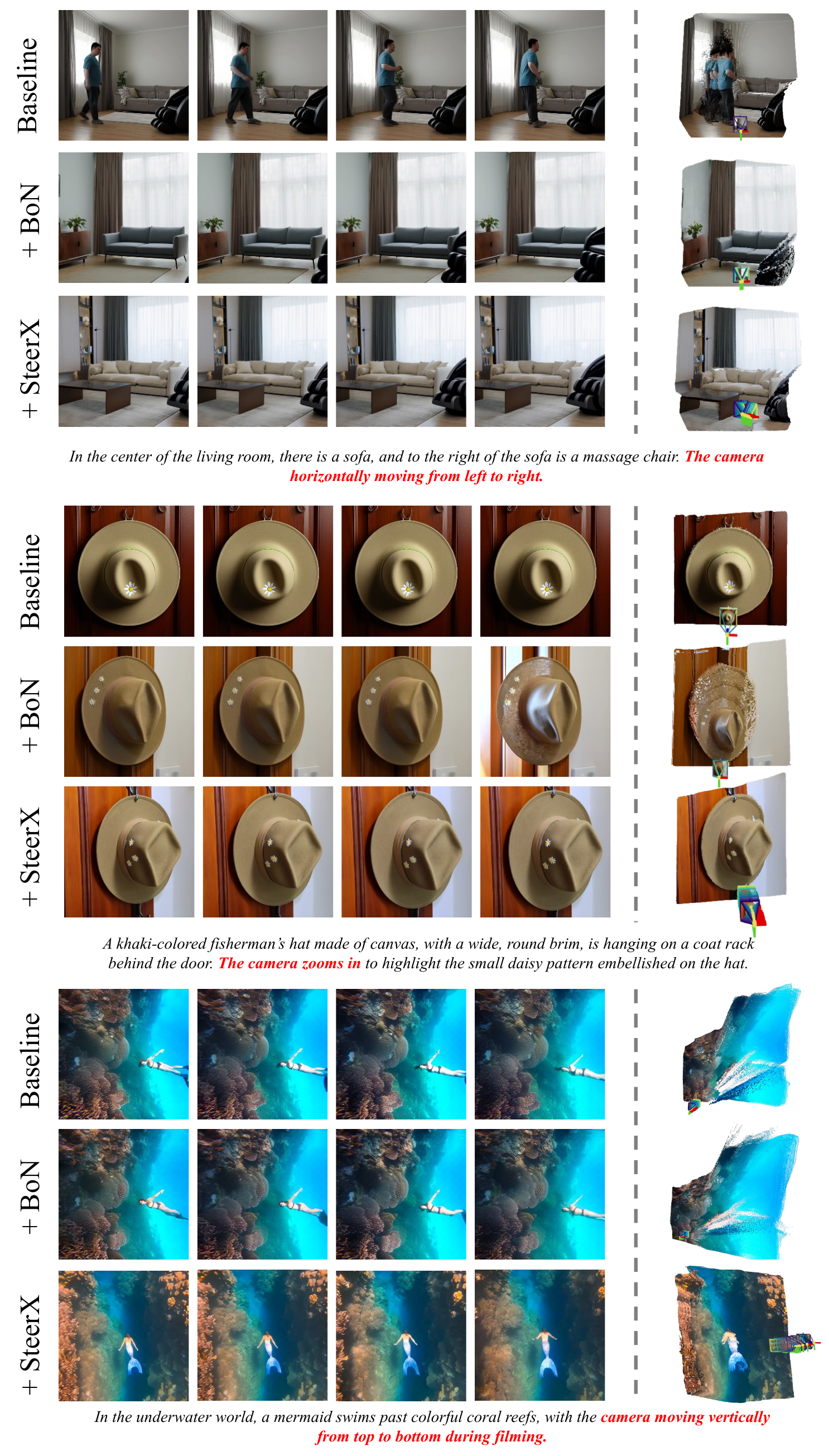}
    \caption{\textbf{Qualitative comparisons on Text-to-4D.}}
    \vspace{\belowfigcapmargin}
    \label{fig:qual_comp_4d}
\end{figure*}